\documentclass{article}
\usepackage{graphicx} 

\usepackage{amsmath, mathrsfs, amssymb, amsthm, fullpage, graphicx, mathtools, bbm, hyperref}
\usepackage[ruled,vlined,linesnumbered]{algorithm2e}

\newtheorem{theorem}{Theorem}
\newtheorem{lemma}{Lemma}

\newtheorem*{example*}{Example}

\newtheorem{remark*}{Remark}

\makeatletter
\def\BState{\State\hskip-\ALG@thistlm}
\makeatother

\newtheorem{proposition}{Proposition}

  {
      \theoremstyle{plain}
      
  }

\numberwithin{equation}{section}

\usepackage{caption}
\usepackage{chngcntr}

\newcommand{\mme}[0]{\mathbb{E}}
\newcommand{\mmp}[0]{\mathbb{P}}
\newcommand{\mmr}[0]{\mathbb{R}}
\newcommand{\mmn}[0]{\mathbb{N}}

\newcommand{\bone}[0]{\mathbbm{1}}

\DeclarePairedDelimiterX{\inp}[2]{\langle}{\rangle}{#1, #2}

\usepackage[nameinlink, noabbrev]{cleveref}

\usepackage{multirow,float}

\usepackage{relsize}

\usepackage{accents}

\newcommand{\footremember}[2]{%
    \footnote{#2}
    \newcounter{#1}
    \setcounter{#1}{\value{footnote}}%
}
\newcommand{\footrecall}[1]{%
    \footnotemark[\value{#1}]%
} 

\title{Sample-Efficient Omniprediction for Proper Losses}
\author{Isaac Gibbs\footremember{UCB}{Department of Statistics, University of California, Berkeley.}\footnote{Email: \href{mailto:igibbs@berkeley.edu}{igibbs@berkeley.edu}.} \and Ryan J. Tibshirani\footrecall{UCB}}
 \date{}

\usepackage{xcolor}
\usepackage{natbib}
\usepackage{xurl}

\allowdisplaybreaks

\begin{document}

\maketitle

\begin{abstract}
    We consider the problem of constructing probabilistic predictions that lead to accurate decisions when employed by downstream users to inform actions. For a single decision maker, designing an optimal predictor is equivalent to minimizing a proper loss function corresponding to the negative utility of that individual. For multiple decision makers, our problem can be viewed as a variant of omniprediction in which the goal is to design a single predictor that simultaneously minimizes multiple losses. Existing algorithms for achieving omniprediction broadly fall into two categories: 1) boosting methods that optimize other auxiliary targets such as multicalibration and obtain omniprediction as a corollary, and 2) adversarial two-player game based approaches that estimate and respond to the ``worst-case" loss in an online fashion. We give lower bounds demonstrating that multicalibration is a strictly more difficult problem than omniprediction and thus the former approach must incur suboptimal sample complexity. For the latter approach, we discuss how these ideas can be used to obtain a sample-efficient algorithm through an online-to-batch conversion. This conversion has the downside of returning a complex, randomized predictor. We improve on this method by designing a more direct, unrandomized algorithm that exploits structural elements of the set of proper losses. 

\end{abstract}

\section{Introduction}

The standard method for fitting a predictive model is to minimize a single loss function measuring its accuracy. In many problems, this framework is employed under the implicit assumption that accurate predictions are sufficient to guide the decisions of downstream users. While this may hold true in some examples, in general, predictive accuracy does not preclude the possibility that the model fails to accurately evaluate the most decision-critical examples. Indeed, classification models trained via empirical risk minimization have frequently been found to be miscalibrated and thus cannot be relied upon to accurately measure outcome uncertainty \citep{Guo2017}. 


In response to this, a growing body of literature has focused on designing predictors that simultaneously satisfy multiple performance criteria.  Rather than solely targeting a low empirical loss, multiaccuracy instead requires the predictor to be unbiased over a collection of reweightings of the covariate space \citep{hebert2018multicalibration, kim2019multiaccuracy}. In applications, these re-weightings often include subgroup indicators and thus multiaccuracy ensures that the predictor remains unbiased across sensitive subpopulations. This is strengthened by multicalibration, which requires the same unbiased criteria to hold conditional on the specific prediction that was issued \citep{hebert2018multicalibration}. Alternatively, another line of work on distributional robustness looks to construct predictors that are simultaneously accurate across a variety of covariate shifts or subpopulations of the data \citep{Mansour2008, Blum2017, Mohri19, Rothblum2021, Duchi2023}. 

In this article, we will focus on constructing predictors that provide simultaneously optimal performance when applied by multiple downstream users to inform decisions. More formally, consider a decision-making task with covariates $X $ and binary outcome $Y \in \{0,1\}$. Let $\hat{p}(X)$ denote an estimate of the conditional probability, $\mmp(Y=1 \mid X)$  that $Y$ is equal to one given $X$ and consider a setting in which a downstream user must use $\hat{p}(X)$ to choose an action $a \in \mathcal{A}$. Given a utility function $u(a,y)$ that characterizes the user's benefit from the action $a$ under true outcome $y$, a natural decision-making procedure is to treat the prediction as though it were perfectly accurate and select an action
\begin{equation}\label{eq:decision_post_processing}
a(\hat{p}(X); u) \in \underset{a \in \mathcal{A}}{\text{argmax}} \mme_{Y' \sim \text{Ber}(\hat{p}(X))}[u(a,Y')],
\end{equation}
that maximizes the expected utility under $Y' \sim \text{Ber}(\hat{p}(X))$. Our goal is to construct predictors that lead to good decisions when applied in this manner by \textit{any} downstream user, i.e., to construct predictors that lead to good performance in (\ref{eq:decision_post_processing}) when applied to arbitrary utility functions. 

Our motivation for this framework comes from practical settings in which a single centralized entity with access to data and statistical expertise must issue predictions that are useful to a diverse array of end users. This type of interaction is common in domains such as weather and epidemiological forecasting in which government organizations regularly issue predictions that are utilized by the general public. Alternatively, one may consider technologies such as language or vision models which are frequently treated as black-boxes by their users. In these settings, the estimated probability $\hat{p}(X)$ could indicate the likelihood that the text or image output by the model contains an error and the user may use this information to decide whether to trust the model or seek out additional assistance.

Without any further restrictions, obtaining optimal decisions in (\ref{eq:decision_post_processing}) is as difficult as exactly learning the true conditional probability function, $p^*(X) := \mmp(Y=1 \mid X)$. Indeed, as we will show in Section \ref{sec:non-param}, the maximum reduction in expected utility that is suffered by taking action $a(\hat{p}(X);u)$ instead of the optimal action, $a(p^*(X);u)$ is directly comparable to the $L_1$ distance between $\hat{p}(X)$ and $p^*(X)$. By standard results in nonparametric estimation, this problem quickly becomes intractable when $X$ is of even moderate dimension (see e.g.~\cite{Stone1982, Devroye1996, Gyorfi2002}). As a result, instead of asking for exact optimal decisions, we will judge $\hat{p}(X)$ by comparing its performance against the best predictor in a restricted class $\mathcal{F}$. More formally, we aim to minimize 
\begin{equation}\label{eq:utility_objective}
\sup_{u : \|u\|_{\infty} \leq 1} \sup_{f \in \mathcal{F}} \mme_{(X,Y)}[u(a(f(X);u),Y)] - \mme_{(X,Y)}[u(a(\hat{p}(X);u),Y) ],
\end{equation}
where the first supremum is over all bounded utility functions\footnote{And, by extension, all possible action spaces.} and the expectations are taken over the test point, $(X,Y)$. Unlike many standard problems in nonparametric estimation, here we place no smoothness assumptions or other restrictions on the distribution of the data. Additionally, it is important to note that in this objective the comparator in $\mathcal{F}$ is allowed to depend on the utility function. On the other hand, the prediction $\hat{p}(X)$ that we construct must be universal to all decision making problems. 

By reformulating (\ref{eq:utility_objective}) slightly, our prediction problem can be seen as a special case of a more general framework known as omniprediction. Introduced by \cite{Goplan2022_Omnipredictors}, omniprediction describes the task of constructing predictors that minimize multiple loss functions simultaneously. Following the above, let $a(p;-\ell) \in \text{argmin}_{a \in \mathcal{A}} \mme_{Y' \sim \text{Ber}(p)}[\ell(a,Y')]$ denote an action in $\mathcal{A}$ that minimizes the loss $\ell$ under $Y' \sim \text{Ber}(p)$. Then, given a set of losses $\mathcal{L}$ and competitor functions $\mathcal{F}$, omniprediction aims to minimize
\begin{equation}\label{eq:omni_obj}
\sup_{\ell \in \mathcal{L}} \sup_{f \in \mathcal{F}} \mme_{(X,Y)}[\ell(a(\hat{p}(X);-\ell),Y) \mid ] -  \mme_{(X,Y)}[\ell(f(X),Y)].
\end{equation}
To connect this to our current setting, let $\ell^u(\hat{p}(X),Y) = -u(a(\hat{p}(X);u),Y)$ denote the loss induced by utility function $u$. It is easy to check that $p \in \text{argmin}_{a \in [0,1]} \mme_{Y \sim \text{Ber}(p)}[\ell^u(a,Y)]$. So, by defining $a(p;-\ell^u) = p$ we obtain the equivalence
\[
\mme[u(a(f(X);u),Y)] - \mme[u(a(\hat{p}(X);u),Y)] = \mme[\ell^u(a(\hat{p}(X);-\ell^u),Y)] - \mme[\ell^u(f(X),Y)] .
\]
As a result, our problem can be equivalently formulated as bounding the omniprediction error (\ref{eq:omni_obj}) with $\mathcal{L} $ taken to be the set of bounded loss functions that are minimized by predicting the true probabilities. In the probabilistic forecasting literature, loss functions with this last property are referred to as \textit{proper} \citep{Gneiting2007}.  

Following the initial work of \citet{Goplan2022_Omnipredictors}, a variety of authors have proposed algorithms for achieving omniprediction. These methods can be broadly categorized into two groups. The first are boosting algorithms \citep{Goplan2022_Omnipredictors, Goplan2023_swapomni, Goplan2023_OI, Globus-Harris2023, Goplan2024_cont_omni, Kim2023}. These methods begin by observing that in order to have low omniprediction error it is sufficient for $\hat{p}(X)$ to satisfy a corresponding set of multiaccuracy, calibration, and/or multicalibration criteria. 
Then, a predictor that satisfies these criteria is constructed in an iterative fashion by identifying and correcting any criterion which is not currently met. The second class of methods are based on algorithms for two-player games \citep{Noarav2025, Garg2024, Okoroafor2025, Lu2025}. Here, the omniprediction problem is framed as a game in which one player constructs a mixture loss that serves as a proxy for the supremum in (\ref{eq:omni_obj}) and the second player constructs the predictor as a best response to this loss. By drawing on tools from the online learning literature, these two players can be designed to guarantee that the predictors returned by the second player satisfy an online form of omniprediction. As shown in \citet{Okoroafor2025} and \citet{Lu2025}, standard online-to-batch conversion methods can then be used to obtain a predictor with low error on i.i.d. data. 

As an aside, we note that a third approach to omniprediction that does not directly use the two-player game set-up, but does draw on closely related tools from the online learning literature, is given in \citet{Dwork2024}. That method is designed specifically for cases in which compositions of the loss and comparator functions can be efficiently embedded in a kernel function class. In general, this embedding leads to suboptimal learning rates for the problems we are interested in and thus we will not focus on this method in detail. 


The remainder of this article is devoted to comparing the sample efficiency of various omniprediction algorithms when applied to the class of proper loss functions. We begin in Section \ref{sec:non-param} by giving a more precise characterization of the omniprediction error when no restrictions are placed on the comparator class. We show that in this case omniprediction is equivalent to $L_1$ estimation of $p^*(X)$ and thus suffers from poor, nonparametric learning rates. Section \ref{sec:cma_omni} considers the performance of boosting methods under the more common setting in which $\mathcal{F}$ has finite VC dimension $\textup{VC}(\mathcal{F}) < \infty$. We show that for a sample of size $n$ the sufficient conditions of multicalibration and calibrated multiaccuracy can be achieved at a rate no better than $\sqrt{\textup{VC}(\mathcal{F})/n} + n^{-2/5}$. Critically, this is strictly worse than the error bound of $\tilde{O}(\sqrt{\textup{VC}(\mathcal{F})/n})$ achieved by two-player game based methods \citep{Okoroafor2025}. Thus, existing boosting methods that target these criteria must be suboptimal. 

It is interesting to note that the error rate achieved by two-player game based methods is (up to polylogarithmic terms) identical to the optimal learning rate for standard risk minimization of a single loss function. Recall that the notation $\tilde{O}(\cdot)$ hides polylogarithmic factors in $\textup{VC}(\mathcal{F})$ and $n$. A classical result in the learning theory literature shows that the best possible error rate for binary classification over the 0-1 loss is $\sqrt{\textup{VC}(\mathcal{F})/n}$  (e.g., Theorem 14.5 of \cite{Devroye1996}). Since the 0-1 loss is proper, this lower bound also applies to our present omniprediction problem. In what follows, we refer to $\sqrt{\textup{VC}(\mathcal{F})/n}$ as the optimal rate for omniprediction and we say that any method that achieves this rate up to polylogarithmic factors is sample efficient.

Sections \ref{sec:simplification} and \ref{sec:algorithms} give our presentation of such sample-efficient methods for omniprediction. Section \ref{sec:simplification} presents a general reduction of the omniprediction problem into the comparatively simpler task of ensembling a finite set of predictors over a small collection of loss functions. Here, we draw heavily on the work of \citet{Savage1971} and \citet{Ehm2016} which demonstrates that all proper losses can be decomposed as mixtures over a class of weighted 0-1 losses. Section \ref{sec:algorithms} then presents two methods. In Section \ref{sec:online_omni}, we discuss two-player game based algorithms and give a new variant of these methods that is simpler to compute. Like all two-player game based methods, this procedure obtains (near) optimal sample complexity, but does so at the cost of producing a complex, randomized predictor. To overcome this shortcoming, in Section \ref{sec:direct_omni} we present a new method that more directly exploits structural properties of the set of proper loss functions to obtain an unrandomized predictor that gives the same optimal error rate.  This partially answers an open question of \citet{Okoroafor2025} who raised the problem of constructing unrandomized predictors that obtain optimal omniprediction error rates.

Empirical comparisons of all the aforementioned algorithms on both simulated examples and a sales forecasting dataset are given in Section \ref{sec:empirics}. As expected, we find that boosting methods give suboptimal performance when compared to the other approaches. On the other hand, methods based on two-player games and our direct ensembling approach realize similar error rates in practice.

While our methods are designed for the binary prediction problem, they can be readily extended to other targets. In Section \ref{sec:other_targets}, we discuss a result of \citet{Steinwart2014} that provides general characterizations of proper losses for other point prediction targets such as conditional means or quantiles. By comparing this result to the binary case, we find that our methods can be applied to construct point predictors that are simultaneously accurate over all proper losses for a given one-dimensional target (e.g., a single mean or quantile). Estimation of multivariate targets is considerably more challenging and provides an interesting open direction for future work.\\




\textbf{Notation:} In what follows, we let $\{(X_i,Y_i)\}_{i=1}^n \subseteq \mathcal{X} \times \{0,1\}$ denote an i.i.d.~training sample. We use $(X,Y)$ to denote a test sample taken independently from the same distribution and $p^*(X) = \mmp(Y=1\mid X)$ to denote the true conditional probability function. Throughout, we will work with the class
\[
\mathcal{L}_0 = \left\{\ell : [0,1] \times [0,1] \to [0,1] \mid \forall p \in [0,1],\ p \in \underset{a \in [0,1]}{\text{argmin}} \mme_{Y' \sim \text{Ber}(p)}[\ell(a,Y')] \right\},
\]
of bounded, proper loss functions. Our goal is to use $\{(X_i,Y_i)\}_{i=1}^n$ to construct a predictor $\hat{p}(X)$ with low omniprediction error, i.e., a low value of
\begin{equation}\label{eq:proper_omni_target}
\sup_{\ell \in \mathcal{L}_0, f \in \mathcal{F}} \mme_{(X,Y)}[\ell(\hat{p}(X),Y) ] -  \mme_{(X,Y)}[\ell(f(X),Y)].
\end{equation}

\section{Comparison to nonparametric estimation}\label{sec:non-param}

To begin understanding the omniprediction problem, it is useful to first consider how (\ref{eq:proper_omni_target}) behaves when $\mathcal{F}$ is allowed to include all possible competitor functions. First, as a sanity check, let us verify that $p^*(X)$ does indeed achieve the minimum possible omniprediction error in this case. Indeed, for any proper loss $\ell$ and predictor $p(X)$,
\[
 \mme[\ell(p^*(X),Y)] =  \mme[\mme[\ell(p^*(X),Y) \mid X ]] \leq \mme[\mme[\ell(p(X),Y) \mid X ]] = \mme[\ell(p(X),Y) ],
\]
where the inequality follows from the definition of propriety. Equivalently, by the same argument $p^*(X)$ is always the optimal predictor for any decision making problem, i.e. for any utility function $u$, 
\[
\mme[u(a(p^*(X) ; u),Y)] \geq \mme[u(a(p(X);u),Y)],
\]
where again this inequality follows by conditioning on $X$ and applying the definition of $a(\cdot)$.

As $\hat{p}(X)$ moves away from $p^*(X)$ it will no longer give optimal performance over all proper losses. This is quantified in the following proposition which shows that for general $\hat{p}(X)$, the maximum performance gap relative to $p^*(X)$ scales with the $L_1$ distance. Since $p^*(X)$ is always the optimal predictor, this proposition can be interpreted as giving bounds on the omniprediction error in the case where no restrictions are placed on $\mathcal{F}$. Proof of this result, along with those of all other results in this paper, can be found in the appendix.

\begin{proposition}\label{prop:non_param_connection}
For any predictor $p : \mathcal{X} \to [0,1]$, 
    \[
\frac{1}{210}\mme[|p(X) - p^*(X)|]^2 \leq \sup_{\ell \in \mathcal{L}_0} \mme[\ell(p(X),Y) ] - \mme[\ell(p^*(X),Y)] \leq 2\mme[|p(X) - p^*(X)|].
\]
\end{proposition}

 It is well known that without heavy parametric assumptions, $L_1$ estimation of $p^*(X)$ suffers from a strong curse of dimensionality. For instance, when $X$ is uniformly distributed on $[-1,1]^k$ and $p^*(X)$ is allowed to be any Lipschitz function, we have the well-known lower bound $\mme[|\hat{p}(X) - p^*(X)|] \geq \Omega(n^{-1/(k+2)})$, where the expectation is taken over both $X$ and the training data $\{(X_i,Y_i)\}_{i=1}^n$ \citep{Stone1982}. One of the key insights of the omniprediction literature is that by placing restrictions on $\mathcal{F}$ we can overcome the curse of dimensionality and recover more tractable rates.

\section{Omniprediction via multicalibration or calibrated multiaccuracy}\label{sec:cma_omni}

Starting with \citet{Goplan2022_Omnipredictors}, a variety of works have considered algorithms for obtaining omniprediction via the stronger notions of multicalibration and calibrated multiaccuracy \citep{Goplan2022_Omnipredictors, Goplan2023_swapomni, Goplan2023_OI, Globus-Harris2023, Goplan2024_cont_omni}. To define these targets formally, let $\mathcal{G}$ denote a class of functions mapping $\mathcal{X}$ to $\mmr$ and $p : \mathcal{X} \to [0,1]$ denote a prediction of $p^*(X)$. We say that $p(X)$ is multicalibrated with respect to $\mathcal{G}$ if 
\[
\mme[g(X)(Y - p(X)) \mid p(X)] \stackrel{a.s.}{=} 0, \forall g \in \mathcal{G}.
\]
We say that $p(\cdot)$ is calibrated if $\mme[ Y \mid p(X)] \stackrel{a.s.}{=} p(X)$ and multiaccurate if 
\[
\mme[g(X)(Y - p(X))] = 0, \forall g \in \mathcal{G}.
\]
Finally, we use the term calibrated multiaccuracy to refer to predictors that are both calibrated and multiaccurate. In essence, multiaccuracy requires the predictor to be unbiased under all re-weightings of the covariate space by functions in $\mathcal{G}$, while calibration asks that the empirical and estimated frequencies of $Y=1$ match over all instances where we make the same prediction. Multicalibration goes further by combining these definitions into a single statement. As a sanity check, one can verify that the true conditional probability function, $p^*(X)$ satisfies all three of these conditions.

Of course, our estimated predictor will never be exactly calibrated or multiaccurate. To measure its discrepancy from these targets, we define the multicalibration, multiaccuracy, and expected calibration errors by 
\begin{equation*}
\begin{gathered}
     \text{MC}(p ; \mathcal{G}) = \sup_{g \in \mathcal{G}} \mme[|\mme[g(X)(Y - p(X)) \mid p(X)] |], \ \text{MA}(p; \mathcal{G}) = \sup_{g \in \mathcal{G}} |\mme[g(X)(Y - p(X)) ]|,\\
      \text{ and } \text{ECE}(p) = \mme[|p(X) - \mme[Y \mid p(X)]|],
\end{gathered}
\end{equation*}
respectively. It is easy to verify that if the constant function $x \mapsto 1$ is in $\mathcal{G}$, the multicalibration error upper bounds both the multiaccuracy and expected calibration errors.

To connect these definitions to omniprediction, we will need to make a specific choice of $\mathcal{G}$. Let $\partial \mathcal{L}_0 = \{p \mapsto \ell(p,1) - \ell(p,0) : \ell \in \mathcal{L}_0\}$ denote the set of discrete derivatives of proper losses and $\partial \mathcal{L}_0\circ \mathcal{F} = \{x \mapsto \ell(f(x),1) - \ell(f(x),0) : \ell \in \mathcal{L}_0\}$ denote the composition of these functions with the comparator class $\mathcal{F}$. Then, \citet{Goplan2023_OI} gives the following bound on the omniprediction error.

\begin{theorem}[Corollary of Lemma 12, Proposition 13, and Theorem 17 in \citet{Goplan2023_OI}]\label{thm:cma_implies_omni} For any predictor $p: \mathcal{X} \to [0,1]$,
\[
\sup_{\ell \in \mathcal{L}_0, f \in \mathcal{F}} \mme[\ell(p(X),Y)  ] - \mme[\ell(f(X),Y)] \leq \textup{MA}(p; \partial \mathcal{L}_0 \circ \mathcal{F})  + \textup{ECE}(p) \leq 2\textup{MC}(p; \partial \mathcal{L}_0 \circ \mathcal{F} \cup \{x \mapsto 1\}) .
\]
\end{theorem}

Despite the extensive study of calibrated multiaccuracy as a vehicle for omniprediction, little is known about the relative difficulty of these two problems beyond Theorem \ref{thm:cma_implies_omni}. As we will now argue, the former is strictly more difficult and necessarily incurs a greater sample complexity. The underlying reason for this comes from two simple high-level observations. First, in order to construct an estimator $\hat{p}(X)$ with low calibration error we must restrict the range of its outputs. In particular, to verify that $|\mme_{(X,Y)}[Y \mid \hat{p}(X) = p] - p|$ is small we need to have many samples for which $\hat{p}(X_i) = p$. This is only possible if $\hat{p}(X)$ takes on only a small number of distinct values. On the other hand, for even very simple function classes, all (approximately) multiaccurate predictors must have sufficient complexity to capture the correlations between $p^*(X)$ and $g(X)$. These two considerations create a natural tension between calibration and multiaccuracy that results in the following lower bound.

\begin{proposition}\label{prop:ca_ma_lower_bound}
    Suppose $\mathcal{X} = \mmr$ and let $\mathcal{G} = \{x \mapsto x\}$ denote the singleton function class containing just the identity. Then,
    \[
    \inf_{\hat{p}} \sup_{P_{XY}}\mme_{\{(X_i,Y_i)\}_{i=1}^n}\left[ \max\{\textup{MA}(\hat{p};  \mathcal{G}), \textup{ECE}(\hat{p})\} \right] \geq cn^{-2/5},
    \]
    where the infimum is over all predictors $\hat{p} : \mathcal{X} \to [0,1]$ estimated using samples $\{(X_i,Y_i)\}_{i=1}^n \stackrel{i.i.d}{\sim} P_{XY}$ and $c > 0$ is a universal constant independent of $n$.
\end{proposition}

Proposition \ref{prop:ca_ma_lower_bound} evaluates calibrated multiaccuracy over a simple singleton function class. To connect this choice of $\mathcal{G}$ with the compositional class $\partial \mathcal{L}_0 \circ \mathcal{F}$ appearing in Theorem \ref{thm:cma_implies_omni}, one may simply note that by taking $\mathcal{F} =\mathcal{G} = \{x \mapsto x\}$ and considering the squared loss we have that $2x - 1 = (x-1)^2 - x^2 \in \mathcal{L}_0 \circ \mathcal{F}$. Using this fact, it is straightforward to argue that Proposition \ref{prop:ca_ma_lower_bound} goes through with $\mathcal{G}$ replaced by $\partial \mathcal{L}_0 \circ \mathcal{F}$ and thus provides a lower bound on the difficulty of calibrated multiaccuracy when applied to omniprediction.

In addition to lower bounding the difficulty of calibration and multiaccuracy in combination, we now also give a lower bound on the difficulty of obtaining multiaccuracy alone. Notably, (up to polylogarithmic factors) this lower bound matches the upper bound previously derived in \citet{Okoroafor2025}.  

\begin{proposition}\label{prop:ma_lower_bound}
    Let $\mathcal{G}$ denote a set of functions of finite VC dimension outputting values in $\{-1,1\}$. Then,
    \[
    \inf_{\hat{p}} \sup_{P_{XY}} \textup{MA}(p;  \mathcal{G}) \geq c\sqrt{\frac{\textup{VC}(\mathcal{G})}{n}},
    \]
    where the infimum is over all predictors $\hat{p} : \mathcal{X} \to [0,1]$ estimated using samples $\{(X_i,Y_i)\}_{i=1}^n \stackrel{i.i.d}{\sim} P_{XY}$ and $c > 0$ is a universal constant independent of $\mathcal{G}$ and $n$.
\end{proposition}

Once again, by choosing $\mathcal{F}$ appropriately it is easy to connect Proposition \ref{prop:ma_lower_bound} to the omniprediction problem. For instance, note that the standard 0-1 loss $\ell(p,y) =  \bone\{p \leq 1/2, y = 1\}  +  \bone\{p > 1/2, y = 0\}$ is proper. If the functions in $\mathcal{F}$ output values in $\{0,1\}$, their composition with the discrete derivative of $\ell$ can be written as  
\[
\ell(f(x),1) - \ell(f(x),0) = \begin{cases}
-1,           & f(x)=1,\\
\phantom{-}1, & f(x)=0.
\end{cases}
\]
and the lower bound of Proposition \ref{prop:ma_lower_bound} also holds with $\mathcal{G}$ replaced by $\mathcal{L}_0 \circ \mathcal{F}$ and $\textup{VC}(\mathcal{G})$ replaced by $\textup{VC}(\mathcal{F})$. 

More generally, by combining the previous two results we find that for $\mathcal{F}$ of finite VC dimension calibrated multiaccuracy cannot be obtained at a rate better than $\sqrt{\textup{VC}(\mathcal{F})/n} + n^{-2/5}$. As we will see shortly, this is strictly worse than the optimal rate of $\sqrt{\textup{VC}(\mathcal{F})/n}$ (up to polylogarithmic factors) for omniprediction. Thus, methods targeting calibrated multiaccuracy and multicalibration cannot possibly produce optimal algorithms for this problem.

To round out our discussion, we conclude this section by giving a new algorithm for calibrated multiaccuracy that obtains an error bound of $\tilde{O}_{\mmp}(\sqrt{\textup{VC}(\mathcal{F})/n} + n^{-1/3})$. This rate is almost identical to our lower bound, which has a slightly larger exponent on the second term, and improves on previous methods for this problem, and for multicalibration, which typically incur sample complexities of order $(\textup{VC}(\mathcal{F})/n)^{1/k}$ for some $k \geq 4$ (e.g. \cite{Goplan2023_OI, Globus-Harris2023, Okoroafor2025}). Unfortunately, the algorithm we present is not computationally efficient due to the fact that it requires looping over all functions in $\mathcal{G}$. Thus, our goal in presenting this result is not to give a new practical method for calibrated multiaccuracy, but rather to help delineate the best rates one can expect for this problem. We leave it as an open problem to close the gap between the upper bound provided by this method and our lower bounds. Finally, we note that while we state this method for finite function classes, it can be readily extended to infinite classes by taking an appropriate cover. 

\begin{proposition}\label{prop:ca_upper_bound}
    Let $\mathcal{G}$ be a finite class of functions outputting values in the bounded range $[-1,1]$. Then, given i.i.d.~samples $\{(X_i,Y_i)\}_{i=1}^n \subseteq \mathcal{X} \times \{0,1\}$, there exists an algorithm that outputs a randomized predictor $\hat{p}(X)$ such that
    \[
    \max\left\{ \textup{MA}(\hat{p} ; \mathcal{G}),  \textup{ECE}(\hat{p}) \right\} \leq \tilde{O}_{\mmp}\left( \sqrt{\frac{\log(|\mathcal{G}|)}{n}} + \frac{1}{n^{1/3}} \right).
    \]
\end{proposition}

At a high-level, our method for achieving calibrated multiaccuracy uses a similar construction to two-player game based algorithms for omniprediction. Namely, it enumerates multiaccuracy and calibration as a list of multiple objectives for $\hat{p}(X)$ and best-responds to mixtures of these objectives in an online fashion. The following section gives a discussion of methods of this type for omniprediction. To avoid duplicating its contents we defer a detailed description of our method for calibrated multiaccuracy to Appendix \ref{sec:app_ca_ma}.

\section{Reduction of omniprediction to finite ensembling}\label{sec:simplification}

In the following section, we will give two methods for obtaining omniprediction at optimal rates. Both of these algorithms will be based on a simplification of the omniprediction problem that replaces the general set of proper losses with a small discrete collection. This allows us to reduce omniprediction to an ensembling task over a finite set of competitors. Precise characterizations of the class of proper loss functions have a long history in the literature dating back to the foundational work of \citet{Savage1971}. In what follows, we will draw in particular on \citet{Ehm2016}.

To begin simplifying the problem, we will first restrict the omniprediction task to the set of losses which are left-continuous in the prediction. This simplification is not critical and in practice we believe it will have little effect on the performance of the predictors. For instance, for a finite action space the decision making function 
\[
a(p;u) \in \underset{a \in \mathcal{A}}{\text{argmax}}\  \mme_{Y' \sim \text{Ber}(p)}[u(a,Y'
)] = \underset{a \in \mathcal{A}}{\text{argmax}}\  p(u(a,1) - u(a,0)) - u(a,0),
\]
is an argmax over a finite collection of linear functions. In particular, this implies that $a(p;u)$ is piecewise constant with discontinuities corresponding to the values of $p$ at which there are multiple optimal actions. To break ties at these points, we may define $a(p;u) = \lim_{p' \uparrow p} \text{argmax}_{a \in \mathcal{A}} \mme_{Y' \sim \text{Ber}(p')}[u(a,Y'
)] $ as the limiting action over smaller values of $p' < p$. One can then verify that with this choice the induced loss $\ell^u(p,y) = -u(a(p;u),y)$ is left-continuous. In general, we believe that this choice of $\ell^u$ is sufficient to capture most practical settings. A short discussion on potential avenues for extending our results to non-left-continuous losses is given in Appendix \ref{sec:app_left_cont_relaxation}. 

In addition to this continuity requirement, we will also restrict ourselves to losses satisfying $\ell(0,0) = \ell(1,1) = 0$. This restriction has no material impact on our results since given an arbitrary proper loss $\ell$ one may always substitute it with the translated loss $\tilde{\ell}(p,y) = \ell(p,y) - \ell(y,y)$ without changing the omniprediction error. In what follows, we use $\mathcal{L}_{\text{lc}}$ to denote the set of losses satisfying the above restrictions. 

Now, our main tool for simplifying $\mathcal{L}_{\text{lc}}$ will be a decomposition of this class in terms of mixtures of weighted 0-1 losses. More precisely, for any $\theta \in [0,1]$ let $\ell_{\theta}$ denote the weighted 0-1 loss given by
\[
\ell_{\theta}(p,y) = \theta \bone\{p > \theta, y = 0\} + (1-\theta) \bone\{p \leq \theta, y = 1\}.
\]
Typically, the term 0-1 loss is used to refer to losses for predicting $y$, not estimating $p^*(X)$. To connect this to the definition above, note that one can view the settings $p > \theta$ and $p \leq \theta$ as predicting that $y=1$ and $y=0$, respectively. The values $\theta$ and $1-\theta$ then determine the relative weights given to errors in each of these predictions.  It is easy to verify that $\ell_{\theta}$ is proper since for any $\tilde{p} \in [0,1]$, 
\begin{equation}\label{eq:ell_theta_proper}
\mme_{Y' \sim \text{Ber}(\tilde{p})}[\ell_{\theta}(p,Y')]  = \theta(1-\tilde{p})\bone\{p > \theta\} + \tilde{p}(1-\theta)\bone\{p \leq \theta\} ,
\end{equation}
and thus the minimizers of the loss are given by 
\[
\underset{p \in [0,1]}{\text{argmin}} \mme_{Y' \sim \text{Ber}(\tilde{p})}[\ell_{\theta}(p,Y')] = \begin{cases}
    [0, \theta) ,\ \tilde{p} < \theta,\\
    (\theta, 1],\ \tilde{p} > \theta,\\
    [0,1],\ \tilde{p} = \theta.
 \end{cases}
\]
In particular, we see that $\tilde{p}$ is always a minimizer. 

The key fact that we will use to simplify the omniprediction problem is the following decomposition of \citet{Ehm2016} which shows that any element of $\mathcal{L}_{\textup{lc}}$ can be obtained as a mixture of these weighted 0-1 losses.

\begin{theorem}[Theorem 1 of \citet{Ehm2016}]\label{thm:proper_loss_decomp} For all $\ell \in \mathcal{L}_{\textup{lc}}$ there exists a non-negative measure $\mu$ on $[0,1]$ such that $\mu([0,1]) \leq 1$ and
\[
\ell(p,y) = \int_{0}^1 \ell_{\theta}(p,y) d\mu(\theta),  \text{ for all } p \in [0,1] \text{ and } y \in \{0,1\}.
\]    
\end{theorem}

Applying Theorem \ref{thm:proper_loss_decomp} to the omniprediction problem we have the equalities,
\begin{align*}
\sup_{\ell \in \mathcal{L}_{\textup{lc}}, f \in \mathcal{F} } \mme_{(X,Y)}[\ell(\hat{p}(X),Y)  ] -  \mme_{(X,Y)}[\ell(f(X),Y)]  & = \sup_{\mu, f \in \mathcal{F}} \int_0^1 \mme_{(X,Y)}[\ell_{\theta}(\hat{p}(X),Y)  - \ell_{\theta}(f(X),Y)] d\mu(\theta)\\
& = \sup_{\theta \in [0,1], f \in \mathcal{F}} \mme_{(X,Y)}[\ell_{\theta}(\hat{p}(X),Y) ] - \mme_{(X,Y)}[\ell_{\theta}(f(X),Y)].
\end{align*}
In particular, we find that the omniprediction error is equal to the maximum error over all weighted $0-1$ losses. To complete our simplification, we will now show that it is sufficient to evaluate this last quantity over $\theta$ falling in a discrete set. 

Fix $m \in \mmn$. Given an arbitrary parameter $\theta \in [0,1]$ our goal will be to round it to the grid $\{\frac{i}{m} - \frac{1}{2m}  : i \in \{1,\dots,m\}\}$. For ease of notation in what follows, let $\theta_i =  \frac{i}{m} - \frac{1}{2m}$. Our first step will be to restrict our predictor to lie on the grid $\{ 0,\frac{1}{m},\frac{2}{m},\dots,1\}$. This restriction is completely innocuous and will be guaranteed by all of the algorithms developed in the subsequent sections. Second, we will assume that the function class $\mathcal{F}$ is closed under constant translations. This assumption is not critical and can be replaced by many other sufficient conditions. The key edge case we need to avoid is one in which there is some predictor $f_{\theta} \in \mathcal{F}$ which is optimal under $\ell_{\theta}$ and whose performance cannot be (approximately) replicated under the rounded loss $\ell_{\theta_i}$ for $\theta_i$ taken to be the value on the grid that is closest to $\theta$. Outside of extreme edge cases, it will typically be the case that $\mme[\ell_{\theta}(f_{\theta}(X),Y)] \approx \mme[\ell_{\theta_i}(f_{\theta}(X),Y)] $ and thus this assumption will not be critical in practice. Under these two restrictions, we have the following simplification of the omniprediction error.

\begin{lemma}\label{lemma:disc_bound}
    Suppose that $\mathcal{F}$ is closed under constant addition. Then, for any predictor $p : \mathcal{X} \to \{ 0,\frac{1}{m},\frac{2}{m},\dots,1\}$,
    \[
    \sup_{\theta \in [0,1], f \in \mathcal{F}} \mme[\ell_{\theta}(p(X),Y)] - \mme[\ell_{\theta}(f(X),Y)] \leq \sup_{i \in \{1,\dots,m\}, f \in \mathcal{F}} \mme[\ell_{\theta_i}(p(X),Y)] - \mme[\ell_{\theta_i}(f(X),Y)] + \frac{1}{m}.
    \]
\end{lemma}

Using this simplification, we will split our methods for constructing $\hat{p}(X)$ into two steps. In the first step, we find predictors $\{\hat{f}_{\theta_i}\}_{i=1}^{m}$ that empirically minimize the losses $\{\ell_{\theta_i}\}_{i=1}^{m}$. If $\mathcal{F}$ is a class of finite VC dimension and $\hat{f}_{\theta_i}$ is the empirical risk minimizer of $\ell_{\theta_i}$ over a sample of size $n$, standard arguments (e.g.~Theorem 6.8 of \citet{SS2014})  guarantee that 
\begin{equation}\label{eq:base_predictor_guarantee}
\sup_{i \in \{1,\dots,m\}, f \in \mathcal{F}} \mme_{(X,Y)}[\ell_{\theta_i}(\hat{f}_{\theta_i}(X),Y)] - \mme_{(X,Y)}[\ell_{\theta_i}(f(X),Y)] \leq O_{\mmp}\left( \sqrt{\frac{\textup{VC}(\mathcal{F})\log(m)}{n}}  \right).
\end{equation}
Then, in the second step we will ensemble $\{\hat{f}_{\theta_i}\}_{i=1}^{m}$ into a single predictor $\hat{p}(X)$ minimizing 
\[
\sup_{i \in \{1,\dots,m\}} \mme_{(X,Y)}[\ell_{\theta_i}(\hat{p}(X),Y) ] - \mme_{(X,Y)}[\ell_{\theta_i}(\hat{f}_{\theta_i}(X),Y)].  
\]
The remainder of this article will be focused on methods for performing this second step. For simplicity in what follows, we will assume that $\{\hat{f}_{\theta_i}\}_{i=1}^{m}$ are fixed in advance and the entire dataset $\{(X_i,Y_i)\}_{i=1}^n$ is
available for ensembling. In practice, and in the application we consider, these predictors will be obtained by splitting the data into two parts, one for fitting $\{\hat{f}_{\theta_i}\}_{i=1}^{m}$ and one for ensembling.

\section{Sample-efficient methods for omniprediction}\label{sec:algorithms}

\subsection{Method based on two-player games}\label{sec:online_omni}

We now present our first of two sample-efficient algorithms for omniprediction. This method is based on a formulation of omniprediction as a two-player game in which one player maintains a mixture over the omniprediction objectives and the other player responds with a predictor that performs well on that mixture. To formalize this, let $q = (q_i)_{i=1}^{m}$ denote a probability distribution over $\{\theta_i\}_{i=1}^{m}$ where $q_i$ denotes the probability of observing $\theta_i$. Consider the mixture over omniprediction objectives given by 
\[
\ell(p,(x,y);q) = \sum_{i=1}^{m} q_i (\ell_{\theta_i}(p,y) - \ell(\hat{f}_{\theta_i}(x),y)).
\]
Following the calculations from the previous sections, in order to guarantee that $\hat{p}(X)$ has small omniprediction error it is sufficient to guarantee that each term in the above sum has a small expected value. The goal of the first player in the game will be to construct a mixture such that 
\begin{equation}\label{eq:q_goal}
\sup_{i \in \{1,\dots,m\}} \mme_{(X,Y)}[\ell_{\theta_i}(\hat{p}(X),Y) - \ell(\hat{f}_{\theta_i}(X),Y) ] \lessapprox \mme_{(X,Y)}[\ell(\hat{p}(X),(X,Y);q) ].
\end{equation}
The goal of the second player is to learn $\hat{p}(X)$ that minimizes the right-hand side. 

In our algorithm, the two players will execute on these objectives in an online fashion. To guarantee (\ref{eq:q_goal}), the first player will use the well-known hedge algorithm, which learns $q$ using online mirror descent over the probability simplex \citep{Vovk1990, Littlestone1994, Freund1997}. In order to respond to $q$, the second player will solve a min-max program that protects against the unknown distribution of $Y \mid X$. More precisely, letting $\Delta_{m}$ denote the set of probability distributions over $\{0,\frac{1}{m},\frac{2}{m},\dots,1\}$, the second player will form its (randomized) prediction at $x$ by solving
\begin{equation}\label{eq:min_max_program}
\min_{P \in \Delta_{m}} \max_{p_y \in [0,1]} \mme_{Y' \sim \text{Ber}(p_y), p \sim P}[\ell(p,(x,Y');q)].
\end{equation}
A critical observation underlying the success of this algorithm is the following bound on the value of this program, which guarantees that the second player always receives a mixture loss of at most zero. 

\begin{lemma}\label{lem:online_omni_min_max_program_bound}
    For any $x \in \mathcal{X}$, 
    \[
    \min_{P \in \Delta_{m}} \max_{p_y \in [0,1]} \mme_{Y' \sim \textup{Ber}(p_y), p \sim P}[\ell(p,(x,Y');q)] \leq 0.
    \]
\end{lemma}
\begin{proof}
    The optimization problem (\ref{eq:min_max_program}) is bilinear in $P$ and $p_y$. Thus,  by von Neumann's min-max theorem \citep{vonNeumann1944} we may swap the order of minimization and maximization to obtain,
    \begin{equation}\label{eq:vn_min_max}
    \min_{P \in \Delta_{m}} \max_{p_y \in [0,1]} \mme_{Y' \sim \text{Ber}(p_y), p \sim P}[\ell(p,(x,Y');q)] =  \max_{p_y \in [0,1]} \min_{P \in \Delta_{m}} \mme_{Y' \sim \text{Ber}(p_y), p \sim P}[\ell(p,(x,Y');q)].
    \end{equation}
    Since each of the losses $\{\ell_{\theta_i}\}_{i=1}^{m}$ are proper, we additionally have that for all $i$,
    \[
    \mme_{Y' \sim \text{Ber}(p_y)}[\ell_{\theta_i}(p_y,Y')] - \mme_{Y' \sim \text{Ber}(p_y)}[\ell_{\theta_i}(\hat{f}_{\theta_i}(x),Y')] \leq 0, 
    \]
    and thus that $ \mme_{Y' \sim \text{Ber}(p_y)}[\ell(p_y,(x,Y');q)] \leq 0$. Moreover, it is easy to check that the value of $\ell_{\theta_i}(p_y,y)$ is unchanged when $p_y$ is rounded to its nearest value on the grid $\{0,\frac{1}{m},\frac{2}{m},\dots,1\}$ (where ties are broken by rounding down). Setting $P$ to be the distribution that puts all its weight on this rounded value in the inner minimization of (\ref{eq:vn_min_max}) gives the desired result. 
\end{proof}

In the implementation of our omniprediction algorithm, we need to solve (\ref{eq:min_max_program}) repeatedly. With only minor modifications, this optimization problem can be written as a linear program over $m$ variables with two constraints corresponding to the values $y \in \{0,1\}$. Optimal solutions for $P$ can then be obtained by calling any standard convex solver. Although this is reasonably computationally efficient, in practice we will typically take $m = \Theta(\sqrt{n})$. While this is not excessively large, it is substantial enough to create a computational burden when solving (\ref{eq:min_max_program}) many times. Fortunately, by exploiting the structure of the $\ell_{\theta}$ losses we can circumvent the need for an off-the-shelf convex solver and instead using the following more direct characterization of the solution. This allows us to solve (\ref{eq:min_max_program}) in $O(m)$ time.

\begin{lemma}\label{lem:min_max_opt_char}
    Fix any $m \in \mmn$, $x \in \mathcal{X}$, and probability distribution $q$. Define the optimal values
    \begin{align*}
    & \theta^* = \sup \left\{\theta \in \left\{0,\frac{1}{m},\frac{2}{m},\dots,1\right\} :  \sum_{i=1}^{m}  q_{i} \bone\{\theta \leq \theta_i \} \geq \sum_{i=1}^{m}  q_{i} \bone\{\hat{f}_{\theta_i}(x) \leq \theta_i\}  \right\}\\
    \text{and } \  & \rho^* =  \frac{\sum_{i=1}^m q_{i} \bone\{\theta^* \leq \theta_i\} - \sum_{i=1}^{m} q_{i} \bone\{\hat{f}_{\theta_i}(x) \leq \theta_i\} }{q_{m\theta^* + 1}},
    \end{align*}
    with the caveat that $\rho^* = 0$ if $\theta^* = 1$. Then, $P^* = (1-\rho^*) \delta_{\theta^*} + \rho^*\delta_{\theta^* + 1/m}$ solves (\ref{eq:min_max_program}). 
\end{lemma}

\begin{algorithm}
\KwData{Data $\{(X_i,Y_i)\}_{i=1}^n$, hyperparameters $m \in \mmn$ and $\eta > 0$, competitor functions $\{\hat{f}_{\theta_i}\}_{i=1}^{m}$.}
$q_{i}(1) = \frac{1}{m}$, for all $i \in \{1,\dots,m\}$\;
\For{$t = 1,\dots,n$ }{
    $\hat{P}_t(x) = \min_{P \in \Delta_{m}} \max_{p_y \in [0,1]}\mme_{Y' \sim \text{Ber}(p_y), p \sim P}[\ell(p,(x,Y');q(t))]$\;
    $\tilde{q}_i(t+1) = q_i(t) \exp(\eta (\mme_{p \sim \hat{P}_t(X_t)}[\ell_{\theta_i}(p,Y_t)] - \ell_{\theta_i}(\hat{f}_{\theta_i}(X_t),Y_t))), \text{ for all } i \in \{1,\dots,m\}$\;
    $q_i(t+1) = \frac{\tilde{q}_i(t+1)}{\sum_{j=1}^m \tilde{q}_j(t+1)}, \text{ for all } i \in \{1,\dots,m\}$\;

}
    \Return $\hat{P} = \frac{1}{n}\sum_{i=1}^n \hat{P}_i $
\caption{Two-player game based omniprediction}
\label{alg:online_omni}
 \end{algorithm}

Algorithm \ref{alg:online_omni} gives a complete description of our two-player game based method for omniprediction. As stated in Theorem \ref{thm:online_omni_bound} below, this method obtains the optimal omniprediction error rate of $\sqrt{\text{VC}(\mathcal{F})/n}$. Formal proof of Theorem \ref{thm:online_omni_bound} is given in Appendix \ref{app:online_omni}. The main idea is to combine Lemma \ref{lem:online_omni_min_max_program_bound} with a regret bound for $q(t)$ that formalizes (\ref{eq:q_goal}) and guarantees that the learned mixture losses are a good proxy for the omniprediction objective. These two results are sufficient to control the online omniprediction error. Generalization to new test samples is then obtained through a standard online-to-batch conversion and the Azuma-Hoeffding inequality.   

\begin{theorem}\label{thm:online_omni_bound}
    Let $\mathcal{F}$ be a function class with finite VC dimension and assume that $\{\hat{f}_{\theta_i}\}_{i=1}^m$ satisfy (\ref{eq:base_predictor_guarantee}). Then, the randomized predictor $\hat{P}$ returned by Algorithm \ref{alg:online_omni} with parameters $m=\Theta(\sqrt{\log(n)/n})$ and $\eta = \Theta(\sqrt{n/\log(m)})$ has omniprediction error bounded as
    \[
    \sup_{\ell \in \mathcal{L}_{\textup{lc}}, f \in \mathcal{F}} \mme_{(X,Y)}[\mme_{p \sim \hat{P}(X)}[\ell(p,Y)]] - \mme_{(X,Y)}[\ell(f(X),Y)] \leq \tilde{O}_{\mmp}\left(\sqrt{\frac{\textup{VC}(\mathcal{F})}{n}} \right).
    \]
\end{theorem}

As discussed in the introduction, we are not the first to propose a method for omniprediction of the form given in Algorithm \ref{alg:online_omni}. \citet{Garg2024} and \citet{Okoroafor2025} both develop two-player game based algorithms that achieve an online omniprediction error of $\tilde{O}(\sqrt{\textup{VC}(\mathcal{F})/n})$. As noted by \citet{Okoroafor2025}, applying an online-to-batch conversion to these procedures then gives an offline omniprediction method with the same error rate. The main contribution of Algorithm \ref{alg:online_omni} relative to these approaches is that it is easier to compute and implement. This largely stems from the fact that we have offloaded the optimization over $\mathcal{F}$ to the first part of our method where we obtain $\{\hat{f}_{\theta_i}\}_{i=1}^{m}$. In contrast, the methods of \citet{Garg2024} and \citet{Okoroafor2025} must perform substantial additional computation to handle the entire set of competitors in $\mathcal{F}$ at each step of the online algorithm. Nevertheless, Algorithm \ref{alg:online_omni} is similar to existing approaches. Our primary goal in this article is not to develop a substantially different two-player game based algorithm, but rather to compare methods of this type to alternative schemes such as those based on calibrated multiaccuracy or the more direct ensembling approach that we will develop next. 

\subsection{Direct ensembling} \label{sec:direct_omni}

In this section we develop a new omniprediction method that more directly exploits the structure of weighted 0-1 losses. Our goal is to overcome some of the shortcomings of two-player game based algorithms. Most critically, the predictor $\hat{P}$ produced by Algorithm \ref{alg:online_omni} is randomized and the only way to compute the distribution of its prediction at $x$ is to solve a large set of $n$ convex optimization problems. These issues are not unique to Algorithm \ref{alg:online_omni} and other two-player game based methods share similar shortcomings \citep{Garg2024, Okoroafor2025}. \citet{Okoroafor2025} raised the open problem of determining if it is possible to achieve low omniprediction error without randomization. Here, we answer this question in the affirmative for proper losses.

\subsubsection{Warm-up: ensembling two predictors}

To motivate our method, it is useful to begin by considering the simplest case in which we just need to ensemble two predictors, $\hat{f}_{\theta_h}(\cdot)$ and $\hat{f}_{\theta_l}(\cdot)$ for associated parameters $\theta_h > \theta_{l}$. Recall that for weighted 0-1 losses there are effectively only two predictions. Namely, given parameter $\theta$ we may either output the prediction $\hat{p}(X) > \theta$ or the prediction $\hat{p}(X) \leq \theta$. The first (resp.~second) prediction is optimal whenever $p^*(X) \geq \theta$ (resp.~$p^*(X) \leq \theta)$. Extending this to the pair of predictions $\hat{f}_{\theta_h}(X)$ and $\hat{f}_{\theta_l}(X)$ we find that there are four possible cases: 
\begin{align*}
& \text{1) } \{\hat{f}_{\theta_h}(X) > \theta_h,\  \hat{f}_{\theta_l}(X) > \theta_l\},\ \text{2) } \{\hat{f}_{\theta_h}(X) \leq\theta_h,\  \hat{f}_{\theta_l}(X) \leq \theta_l\}, \\
& \text{3) } \{\hat{f}_{\theta_h}(X) \leq \theta_h,\ \hat{f}_{\theta_l}(X) > \theta_l\},\  \text{4) } \{\hat{f}_{\theta_h}(X) > \theta_h,\  \hat{f}_{\theta_l}(X) \leq \theta_l\}.
\end{align*}
In the first three cases, the predictions of $\hat{f}_{\theta_h}(X)$ and $\hat{f}_{\theta_l}(X)$ are consistent with each other and to obtain a small omniprediction error we may simply define $\hat{p}(X)$ to agree with both of them. In particular, in case one we can set $\hat{p}(X) > \theta_h > \theta_l$, in case two we can set $\hat{p}(X) \leq \theta_l < \theta_{h}$ and in case three we can set $\theta_l < \hat{p}(X) \leq \theta_h$. On the other hand, in case four the predictions of $\hat{f}_{\theta_h}(X)$ and $\hat{f}_{\theta_l}(X)$ are contradictory. To resolve this disagreement, we can examine the data and set
\[
\hat{p}(X) \in \begin{cases}
    (\theta_h, 1],\ \hat{\mmp}_n(Y \mid \hat{f}_{\theta_h}(X) > \theta_h, \hat{f}_{\theta_l}(X) \leq \theta_l) > \theta_h,\\
    (\theta_l, \theta_h],\ \hat{\mmp}_n(Y \mid \hat{f}_{\theta_h}(X) > \theta_h, \hat{f}_{\theta_l}(X) \leq \theta_l) \in (\theta_l, \theta_h],\\
    [0,\theta_l],\ \hat{\mmp}_n(Y \mid \hat{f}_{\theta_h}(X) > \theta_h, \hat{f}_{\theta_l}(X) \leq \theta_l) \leq \theta_l,
\end{cases}
\]
where $\hat{\mmp}_n$ denotes the empirical probability over $\{(X_i,Y_i)\}_{i=1}^n$. As the following lemma verifies, this definition produces low omniprediction error.
\begin{lemma}\label{lem:two_pred_ensemble}
    Fix any $\theta_h > \theta_l$ and predictors $\hat{f}_{\theta_h}(\cdot)$ and $\hat{f}_{\theta_l}(\cdot)$. Let $\hat{p}(X)$ be defined as above. Then,
    \[
    \max_{\theta \in \{\theta_l,\theta_h\}} \mme_{(X,Y)}[\ell_{\theta}(\hat{p}(X),Y) ] - \mme_{(X,Y)}[\ell_{\theta}(\hat{f}_{\theta}(X),Y)] \leq O_{\mmp}\left(\sqrt{\frac{1}{n}} \right)
    \]
\end{lemma}
\begin{proof}
    For simplicity, we will only consider the case $\theta = \theta_l$. The case $\theta = \theta_h$ is identical. Let $E = \{\hat{f}_{\theta_h}(X) > \theta_h, \hat{f}_{\theta_l}(X) \leq \theta_l\}$ denote the event where the predictors disagree. By construction, we have  
    \begin{align*}
    & \mme_{(X,Y)}[\ell_{\theta_l}(\hat{p}(X),Y) ] -  \mme_{(X,Y)}[\ell_{\theta_l}(\hat{f}_{\theta_l}(X),Y)]  = \mme_{(X,Y)}[(\ell_{\theta_l}(\hat{p}(X),Y) -\ell_{\theta_l}(\hat{f}_{\theta_l}(X),Y)) \bone\{E\} ]\\
     & = \mme[(\ell_{\theta_l}(1,Y) - \ell_{\theta_l}(0,Y)) \bone\{E\}] \bone\{\hat{\mmp}_n[Y \mid E] > \theta_l\}\\
     & =  \mme[(\ell_{\theta_l}(1,Y) - \ell_{\theta_l}(0,Y)) \bone\{E\}] \bone\{\hat{\mme}_n[(\ell_{\theta_l}(1,Y) - \ell_{\theta_l}(0,Y)) \bone\{E\}] \leq 0\},
    \end{align*}
    where the last equality follows from the definition of $\ell_{\theta_l}$. This last quantity can be bounded using Hoeffding's inequality.
\end{proof}

\subsubsection{General case}

Extending Lemma \ref{lem:two_pred_ensemble} beyond two predictors requires considerable care. Recall that our goal is to ensemble the $m$ estimators $\{\hat{f}_{\theta_i}\}_{i=1}^{m}$. These functions can make a total of $2^{m}$ different combinations of predictions the vast majority of which contain some disagreements. Notably, we cannot obtain accurate estimates of the true probability of $Y=1$ under all of these combinations simultaneously. As a result, instead of evaluating these events individually, we will use an iterative scheme in which the predictors are ensembled in groups. 

The main primitive in these iterations is a merge algorithm that takes as input two predictors $\hat{p}_h(X)$ and $\hat{p}_{l}(X)$ which are designed to give low error on losses $\ell_{\theta}$ for $\theta \in \Theta_h$ and $\theta \in \Theta_l$, respectively. The sets $(\Theta_h, \Theta_l)$ are constructed so that $\theta_h > \theta_l$ for all $\theta_h \in \Theta_h$ and $\theta_l \in \Theta_l$.  The output of this method will be a single predictor, $\hat{p}_{m}(X)$ that obtains loss comparable to $\hat{p}_h(X)$ on all parameters $\theta_h \in \Theta_{h}$ and loss comparable to $\hat{p}_{l}(X)$ on all parameters $\theta_l \in \Theta_l$.

As expected, the main issue in this merge procedure is resolving disagreements between $\hat{p}_h(X)$ and $\hat{p}_{l}(X)$. This is done using the following iterative scheme. First, we begin by simply positing that $\hat{p}_h(X)$ is a good predictor and setting $\hat{p}(X) = \hat{p}_h(X)$. This immediately guarantees that $\hat{p}(X)$ has good performance on $\Theta_h$, but leaves open the possibility that it fails on one of the parameters in $\Theta_l$. To address this, we iterate through the parameters $\theta_l \in \Theta_l$ in descending order and examine each of the empirical expectations, $\hat{\mme}_n[(\ell_{\theta_l}(1,Y)- \ell_{\theta_l}(0,Y))\bone \{X \in E\}]$ where $E = \{x : \hat{p}_h(x) > \min_{\theta_h \in \Theta_j} \theta_h,\ \hat{p}_l(x) \leq  \theta_l\}$ is the set where the two predictors disagree. If this expectation is negative it means that predicting a high value gives a low loss and thus $\hat{p}(X)$ will be guaranteed to give good performance on $\ell_{\theta_l}$. On the other hand, if it is positive, then we need to predict a small value. To account for this, we modify our predictor so that $\hat{p}(x) = \hat{p}_l(x)$ for all $x \in E$. Notably, due to the hierarchical structure of weighted 0-1 losses, this single modification will be sufficient to guarantee that $\hat{p}(X)$ is a good predictor on all previously considered parameters $\theta \in \Theta_{l}$ with $\theta > \theta_l$. This follows immediately from the fact that for any such $\theta$,
\begin{align*}
\hat{\mme}_n[(\ell_{\theta}(1,Y)- \ell_{\theta}(0,Y)) \mid X \in E]   = \theta - \mmp(Y = 1 \mid  X \in E)  & \geq \theta_l - \mmp(Y = 1 \mid  X \in E) \\
& = \hat{\mme}_n[(\ell_{\theta_l}(1,Y)- \ell_{\theta_l}(0,Y)) \mid X \in E] > 0.
\end{align*}
However, it may now give poor performance on some losses in $\Theta_{h}$. This is corrected by performing a similar set of iterations over the parameters in $\Theta_h$. Eventually, after repeating this entire process many times we will have evaluated all parameters in $\Theta_h$ and $\Theta_{l}$ and certified the performance of $\hat{p}(X)$ on each of them.

\begin{algorithm}[ht]
\KwData{Predictors $\hat{p}_l$, $\hat{p}_{h}$, optimality sets $\Theta_{h} > \Theta_l$, data $\{(X_i,Y_i)\}_{i=1}^n$, and hyperparameter $\epsilon$.}
$\hat{p}_{m} = \hat{p}_h$\;
$\theta_h = \min \Theta_h$\;
$\theta_l = \max \Theta_l$\;
$\textup{dir} = \textup{low}$\;
\While{$\theta_l \neq -\infty$, $\theta_{h} \neq \infty$}{
$E = \{x : \hat{p}_{h}(x) > \theta_l, \hat{p}_l(x) \leq \theta_l\}$\;
\uIf{$\textup{dir} = \textup{low}$}{
    \uIf{$\hat{\mme}_n[(\ell_{\theta_l}(1,Y)- \ell_{\theta_l}(0,Y))\bone\{X \in E\}] <- \epsilon$ }{
        $\hat{p}_{m}(x) $ = $\hat{p}_l(x)$, for all $x \in E$\;
        $\theta_{h} = \min\{\theta \in 
        \Theta_{h} : \theta > \theta_{h}\}$\;
        $\textup{dir} = \textup{high}$\;
    }\Else{
    $\theta_l = \max\{\theta \in 
        \theta_l : \theta < \Theta_l\}$\;
    }
}\Else{ \tcp{Do a symmetric set of iterations through $\Theta_h$ in which we alter the value}
\tcp{of $\hat{p}_m(X)$ if $\hat{\mme}_n[(\ell_{\theta_l}(1,Y)- \ell_{\theta_l}(0,Y))\bone\{X \in E\}] > \epsilon$.}
    
}
}
    \Return $\hat{p}_{m} $
\caption{Merge}\label{alg:merge}
 \end{algorithm}

Algorithm \ref{alg:merge} gives a summary of the merge method, a more detailed description of which can be found in Appendix \ref{app:direct_omni}. In total, this algorithm will evaluate each element of $\Theta_h$ and $\Theta_{l}$ at most once and thus will be guaranteed to run in at most $|\Theta_h| + |\Theta_{l}|$ iterations. In addition to the description given above, Algorithm \ref{alg:merge} contains one additional hyperparameter, $\epsilon$ that gives a buffer on the improvement in the loss that must be observed before swapping $\hat{p}_{m}(X)$ between $\hat{p}_h(X)$ and $\hat{p}_l(X)$. In our theoretical results, correct specification of this hyperparameter is necessary in order to mitigate the sensitivity of $\hat{p}_{m}(X)$ to noise and ensure its generalization to new data. In general, ensuring the generalization of iterative schemes of this type is a difficult problem and the approach we take here is partially inspired by the work of \citet{Deng2024} which uses a similar buffer hyperparameter in a different context. On the other hand, in our experiments we find that this hyperparameter is not crucial and the lowest omniprediction error is achieved when $\epsilon = 0$. As a result, we will not place a heavy emphasis on the choice of $\epsilon$.

Lemma \ref{lem:merge_performance} states our formal guarantee on the omniprediction error of the merge procedure. In this lemma, we assume that the values of $\hat{p}_h(X)$ and $\hat{p}_l(X)$ are restricted to $(\max \Theta_l,1]$  and $[0,\min \Theta_h)$, respectively. The idea here is that $\hat{p}_h(X)$ (resp.~$\hat{p}_l(X)$) only gives information about the parameters in $\Theta_h$ (resp.~$\Theta_l)$ and does not give any signal about $\Theta_l$ (resp.~$\Theta_{h}$). In our applications of the merge procedure this assumption will be guaranteed to hold by construction. 

\begin{lemma}\label{lem:merge_performance}
     Let $\Theta_h > \Theta_l$ be finite subsets of $[0,1]$ and assume that $\hat{p}_h(X)$ takes values in $(\max \theta_l,1]$  and $\hat{p}_l(X)$ take values in $[0,\min \Theta_h)$.  Then, the predictor $\hat{p}_{m}(X)$ returned by Algorithm \ref{alg:merge} with $\epsilon = \Theta(\sqrt{\log(|\Theta_h| + |\Theta_l|)/n})$ has omniprediction error,
    \[
    \sup_{a \in \{h,\ell\}} \sup_{\theta \in \Theta_a} \mme_{(X,Y)}[\ell_{\theta}(\hat{p}_{m}(X),Y) ] - \mme_{(X,Y)}[\ell_{\theta}(\hat{p}_{a}(X),Y)] \leq O_{\mmp}\left(\sqrt{\frac{\log(|\Theta_h|) + \log(|\Theta_l|)}{n}} \right)
    \]
\end{lemma}

With this merge procedure in hand, ensembling the larger collection of predictors $\{\hat{f}_{\theta_i}\}_{i=0}^m$ is relatively straightforward. Namely, we simply apply the merge procedure repeatedly by joining together predictors with adjacent parameters until we are left with only a single function. Concretely, assume that $m = 2^k$ is a power of $2$. Then, we will proceed in $k$ rounds, where in each round adjacent predictors are paired up and then merged (e.g. in round 1 we merge the pairs $(\hat{f}_{\theta_1},\hat{f}_{\theta_2}),\dots,(\hat{f}_{\theta_{m-1}}, \hat{f}_{\theta_{m}})$). In order to guarantee the generalization of this method theoretically, each of these $k$ rounds will use fresh data. This is specified on line 3 of Algorithm \ref{alg:ensemb_main}, where we use the notation $\text{Split}(\{(X_i,Y_i)\}_{i=1}^n)$ to denote a division of the training data into $\log_2(m)$ (approximately) equally sized folds. Here, data splitting ensures that the empirical expectations appearing in the merge procedure stay uniformly close to their population counterparts. In practice, we find that this is unnecessary and all of the data can be used at every round without issue.

 \begin{algorithm}
 \KwData{Predictors $(\hat{f}_{\theta_i})_{i=1}^{m}$, data $\{(X_i,Y_i)\}_{i=1}^n$, hyperparameter $\epsilon$.}
 $(\hat{p}_{1,i})_{i=1}^{m} = (\hat{f}_{\theta_i})_{i=1}^{m}$\;
$(\Theta_{1,i})_{i=1}^{m} = (\{\theta_i\})_{i=1}^{m}$ \tcp*{$\hat{p}_{t,i}$ is designed to be "optimal" on $\Theta_{t,i}$}
$D_1,\dots,D_{\log_2(m)} = \text{Split}(\{(X_i,Y_i)\}_{i=1}^n)$ \tcp*{\hspace{-0.2cm}Split the data into (approximately) equal parts}
      \For{$t =1,\dots,\log_2(m)$}{
      \For{$i = 1,\dots,\frac{m}{2^{t}}$}
      {
        $\hat{p}_{t+1,i} = \textup{Merge}(\hat{p}_{t,2i-1},\hat{p}_{t,2i}, \Theta_{t,2i-1},\Theta_{t,2i}, D_t,\epsilon)$\;
        $\Theta_{t+1,i} = \Theta_{t,2i-1} \cup \Theta_{t,2i}$.
     }
      }
    \Return $\hat{p} = \hat{p}_{\log_2(m),1}$
\caption{Ensembling Scheme}\label{alg:ensemb_main}
 \end{algorithm}

Algorithm \ref{alg:ensemb_main} states our method formally. In this algorithm, and in what follows, we will assume that $\hat{f}_{\theta_i}$ takes values in $\{\theta_i - 1/(2m),\theta_i + 1/(2m)\}$. This is always possible since given an arbitrary predictor $\tilde{f}_{\theta_i}$ with good performance under $\ell_{\theta_i}$ we may always recode its predictions as 
\[
\hat{f}_{\theta_i}(X) = (\theta_i - 1/(2m))\bone\{\tilde{f}_{\theta_i}(X) \leq \theta_i\} +  (\theta_i + 1/(2m))\bone\{\tilde{f}_{\theta_i}(X) > \theta_i\}.
\]
As above, the idea is that $\hat{f}_{\theta_i}(\cdot)$ only provides information on whether $p^*(X)$ lies above or below $\theta_i$. The following theorem shows that this method achieves the optimal omniprediction error rate. 

\begin{theorem}\label{thm:direct_omni_bound}
    Let $\mathcal{F}$ be a function class with finite VC dimension and assume that $\{\hat{f}_{\theta_i}\}_{i=1}^m$ satisfy (\ref{eq:base_predictor_guarantee}). Then, the predictor $\hat{p}(\cdot)$ returned by Algorithm \ref{alg:ensemb_main} with $m = \Theta(2^{\lfloor \log_2(\sqrt{n}) \rfloor})$ and $\epsilon = \Theta(\sqrt{\log(n)/n})$ has omniprediction error bounded as
    \[
    \sup_{\ell \in \mathcal{L}_{\textup{lc}}, f \in \mathcal{F}} \mme_{(X,Y)}[\ell(\hat{p}(X),Y) ] - \mme_{(X,Y)}[\ell(f(X),Y)] \leq \tilde{O}_{\mmp}\left( \sqrt{\frac{\textup{VC}(\mathcal{F})}{n}} \right).
    \]
\end{theorem}

\section{Empirical Comparisons}\label{sec:empirics}

We now turn our attention to a set of empirical comparisons. Following the previous sections, we will evaluate three methods for omniprediction:
\begin{itemize}
    \item \textbf{CalMA:} Our first method is the calibrated multiaccuracy procedure proposed in Algorithm 2 of \citet{Goplan2023_OI}. This method uses a boosting approach that iteratively updates $\hat{p}(X)$ by alternating between improving its multiaccuracy error and improving its calibration error. We will implement this algorithm so that it targets multiaccuracy with respect to the function class $\{\ell(\hat{f}_{\theta_i}(\cdot),1) - \ell(\hat{f}_{\theta_i}(\cdot),0) : i \in \{1,\dots,m\}\}$. A straightforward extension of Theorem \ref{thm:cma_implies_omni} shows that this (combined with calibration) is sufficient to give low omniprediction error.
    
    The calibrated multiaccuracy procedure of \citet{Goplan2023_OI} contains a hyperparameter, $\alpha$ that specifies the target omniprediction error. The theory presented in that work suggests that this parameter should be chosen to be of order $ \alpha = \Theta((\log(m)/n)^{-1/4} + n^{-1/10})$. We find that this is needlessly pessimistic and will prefer to take $\alpha = c\sqrt{\log(m)/n}$ for some constant $c$ that we vary. 
    In addition, the theory for this method requires extensive data splitting in order to ensure that fresh samples are available for each of up to $O(1/\alpha^2)$ iterations of the algorithm. For the sample sizes we consider, this would give us only a handful of data points at each iteration with which to correct the multiaccuracy and calibration error. As this is clearly impractical, we do not perform any data splitting and simply use all available data at every step. As we will see shortly, this does not appear to be an issue and the algorithm gives reasonable empirical performance.
    \item \textbf{Two-player:} Our second algorithm is the two-player game based procedure given in Algorithm \ref{alg:online_omni}. We implement this method with hyperparameter $\eta = c\sqrt{\log(m)/n}$ for varying levels of $c$. 
    \item \textbf{Direct ensembling:} Our final method is the direct ensembling procedure proposed in Algorithm \ref{alg:ensemb_main}. Similar to the previous methods, we implement this procedure with parameter $\epsilon = c\sqrt{\log(m)/n}$ for varying levels of $c$. Additionally, as above, we do not utilize data splitting. We find that although our theoretical results require fresh data for every round of merging, in practice this method offers robust performance when all the available data is used at each step. 
\end{itemize}

All methods are implemented with the same value of $m$ and the same set of initial predictors $\{\hat{f}_{\theta_i}\}_{i=1}^{m}$. The exact procedure for obtaining these quantities varies for each experiment and is specified in the relevant sections. 

\subsection{Simulated example}

For our first example, we consider a simple simulated dataset that illustrates the core ensembling problem. Let $\mathcal{F} = \{x \mapsto \beta_0 + \beta_1 x : \beta_0, \beta_1 \in \mmr\}$ be the class of linear predictors on $\mmr$. Let $X$ be supported on $\{0.05,0.45,0.85\}$ with distribution $\mmp(X = 0.05) = 0.1$, $\mmp(X = 0.45) = 0.6$, $\mmp(X = 0.85) = 0.3$ and let $Y \in \{0,1\}$ be sampled according to $\mmp(Y \mid X = 0.05) = 0.3$, $\mmp(Y \mid X = 0.45) = 0.9$, and $\mmp(Y \mid X = 0.85) = 0.4$. By design, this distribution for $(X,Y)$ has the property that the optimal linear predictor $f^*_{\theta} \in \mathcal{F}$ under loss $\ell_{\theta}$ gives inconsistent predictions as $\theta$ varies. For example, at $\theta = 0.35$ and $X = 0.05$ the optimal predictor outputs $f^*_{0.35}(0.05) \leq 0.35$, while at $\theta = 0.75$ it predicts $f^*_{0.75}(0.05) > 0.75$. This inconsistency in the optimal predictions is illustrated in the left panel of Figure \ref{fig:simulated_example}, which plots the conditional distribution of $Y$ given $X$ alongside these optima. 

\begin{figure}
    \centering
    \includegraphics[width=\linewidth]{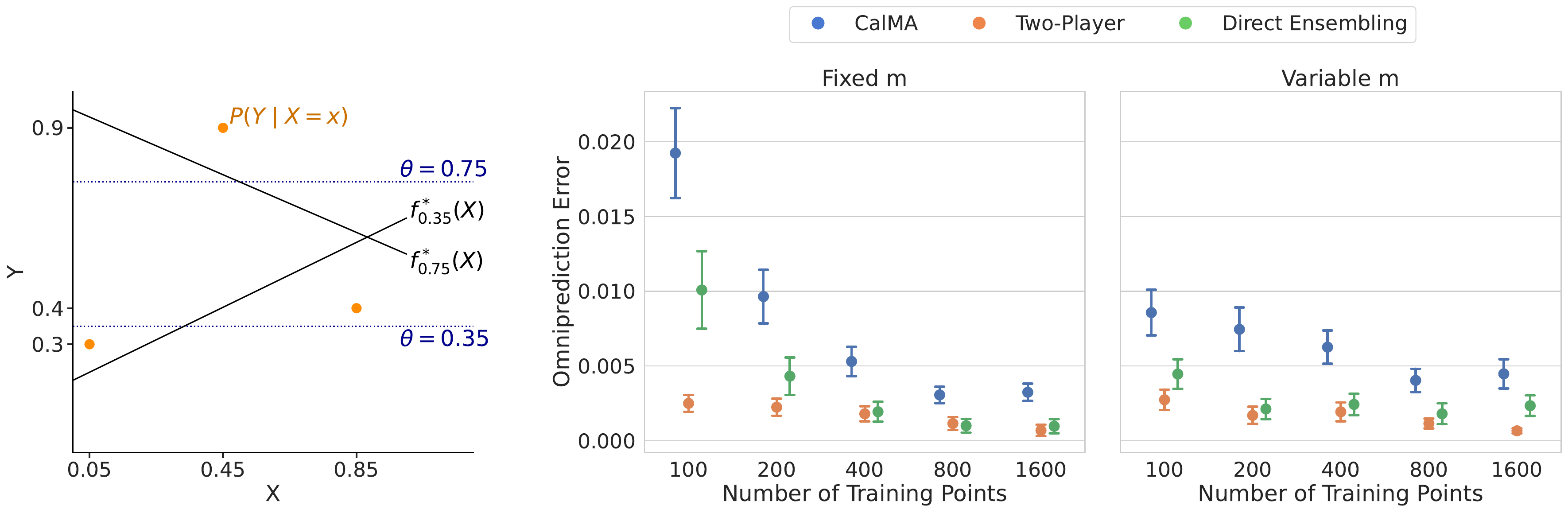}
    \caption{Illustration of the core ensembling problem for our simulated example (left panel) and realized average omniprediction error of the calibrated multiaccuracy (blue), two-player game based (orange), and direct ensembling (green) methods across various sample sizes with $m = 16$ fixed (center panel) or chosen variably as $m = 2^{\lfloor \log_2(\sqrt{n}) \rfloor}$ (right panel). Dots and error bars show means and standard errors obtained by evaluating the omniprediction error over 2000 test points for each of 40 draws of the training data. Hyperparameters for the calibrated multiaccuracy, two-player, and direct ensembling methods are set as $c = 0.5$, $c=32$, and $c=0$, respectively.}
    \label{fig:simulated_example}
\end{figure}

The rightmost two panels of Figure \ref{fig:simulated_example} compare the performance of the three main omniprediction methods over various sample sizes and settings of $m$. To simplify our initial comparisons, results in this figure show only a single hyperparameter setting for each method which was found to give good performance. Dots and error bars display empirical estimates of the average omniprediction error,
\[
\mme\left[ \sup_{i \in \{1,\dots,m\}} \mme_{(X,Y)}[\ell_{\theta_i}(\hat{p}(X),Y)] - \mme_{(X,Y)}[\ell_{\theta_i}(\hat{f}_{\theta_i}(X),Y) ] \right],
\]
over multiple draws of the training dataset. The center panel shows results for a fixed value of $m=16$ while the right panel gives results for $m = 2^{\lfloor \log_2(\sqrt{n}) \rfloor}$ increasing with the sample size. In both cases, the initial predictors $\{\hat{f}_{\theta_i}\}_{i=0}^m$ are obtained by solving the mixed integer programs, 
\[
\min_{f \in \mathcal{F}} \frac{1}{k} \sum_{j=1}^k \ell_{\theta_i}(f(\tilde{X}_j),\tilde{Y}_j),
\]
over an independent dataset $\{(\tilde{X}_j,\tilde{Y}_j)\}_{j=1}^{k}$ of size $k = 500$. 

\begin{figure}[ht]
    \centering
    \includegraphics[width = \textwidth]{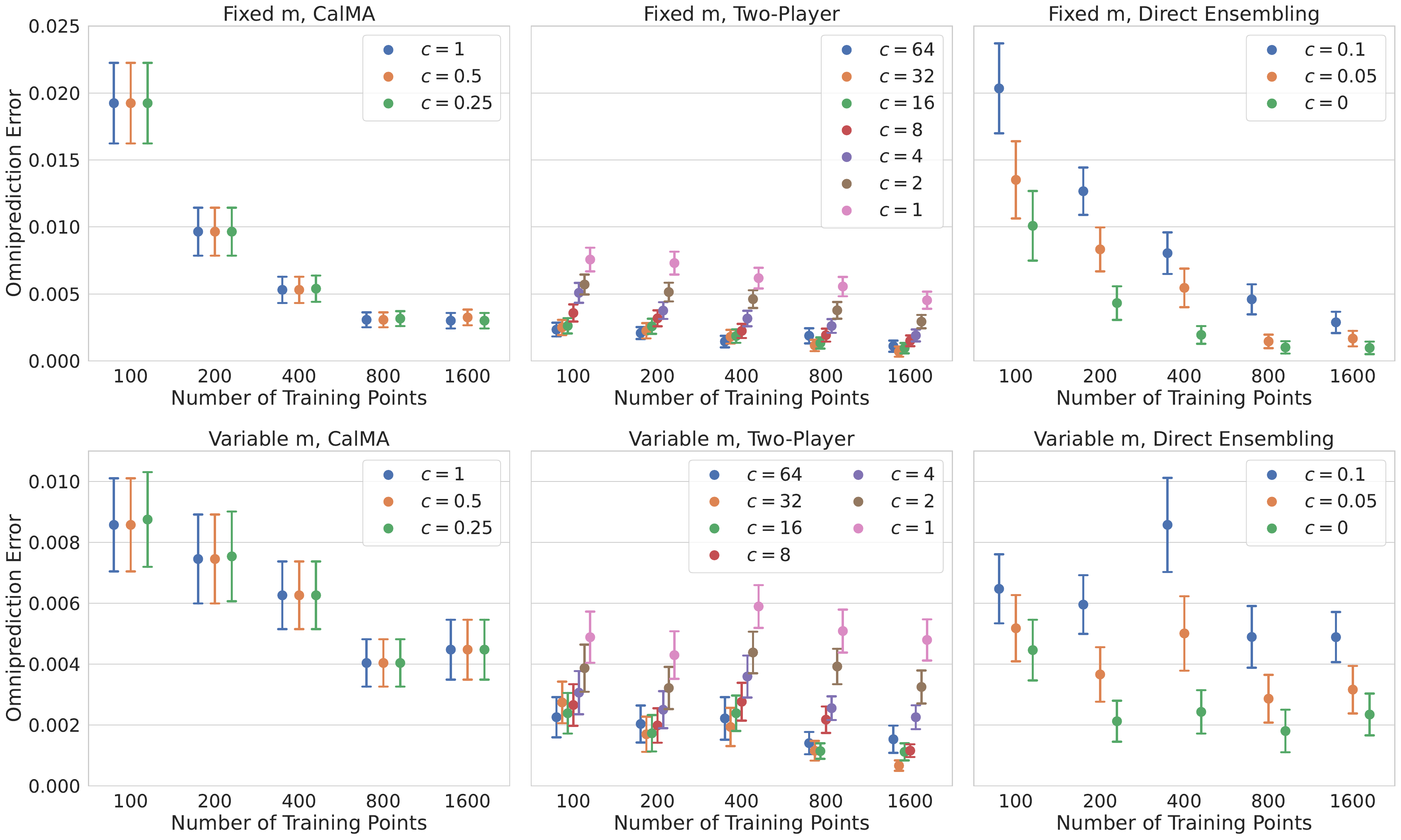}
    \caption{Omniprediction error of the calibrated multiaccuracy (left panels), two-player game based (center panels), and direct ensembling (right panels) methods across various sample sizes with $m = 16$ fixed (top row) or chosen variably as $m = 2^{\lfloor \log_2(\sqrt{n}) \rfloor}$ (bottom row) as the scaling constant $c$ varies on a simulated dataset. Dots and error bars show means and standard errors obtained by evaluating the omniprediction error over 2000 test points for each of 40 draws of the training data.}
    \label{fig:one_d_param_vary}
\end{figure}

Overall, we find that, as expected, the method based on calibrated multiaccuracy realizes the highest omniprediction error across all sample sizes and settings of $m$. On the other hand, the two-player game based method performs better than the direct ensembling procedure at smaller sample sizes, while the two methods obtain nearly identical performance at larger values of $n$. An advantage of the direct ensembling approach is that it offers simplified hyperparameter tuning. Figure \ref{fig:one_d_param_vary} displays results for the three methods as the scaling constant $c$ varies. We find that the direct ensembling method always performs best with parameter $\epsilon = 0$. On the other hand, to obtain good performance with the two-player game based approach we must choose an intermediate value of $\eta$. In practice, selecting such a value may be challenging and could require additional data splitting.

\subsection{Sales forecasting}\label{sec:sales}

Our second experiment compares the three omniprediction methods on a retail sales forecasting dataset taken from the M5 forecasting challenge \citep{Makridakis2022}. In this challenge, competitors were tasked with constructing quantile forecasts of the daily sales of various items at ten different Walmart stores over a 28-day period. Here, we transform this task to a binary prediction problem in which the goal is to estimate the probability that at least one unit of an item is sold at a given store on a given day. To do this, we use linear interpolation to convert the quantile forecasts given by the competitors into estimates of the full cumulative distribution function of the sales. We then set our function class $\mathcal{F}$ to be corresponding forecasts of the probability that at least one sale is made. Details of this procedure are given in Appendix \ref{sec:app_sales_details}. In total, the M5 dataset contains quantile forecasts from the top 50 participants in the competition. To obtain a sufficient sample size for our experiments, here we restrict to the 43 forecasters who issued predictions for at least 10000 product-store pairs on day 7.

We evaluate the omniprediction methods in three steps. First, to obtain $\{\hat{f}_{\theta_i}\}_{i=1}^m$ we randomly select 500 product-store pairs from the day 7 data. Then, for each $i \in \{1,\dots,m\}$ we set $\hat{f}_{\theta_i}$ to be the element of $\mathcal{F}$ that minimizes the empirical loss, $\ell_{\theta_i}$ over these 500 samples. With these initial predictors in hand, we run the three omniprediction methods on a randomly chosen subset of the data from day 14. Finally, all methods are evaluated on the data from day 21.

\begin{figure}[ht]
    \centering
    \includegraphics[width=\linewidth]{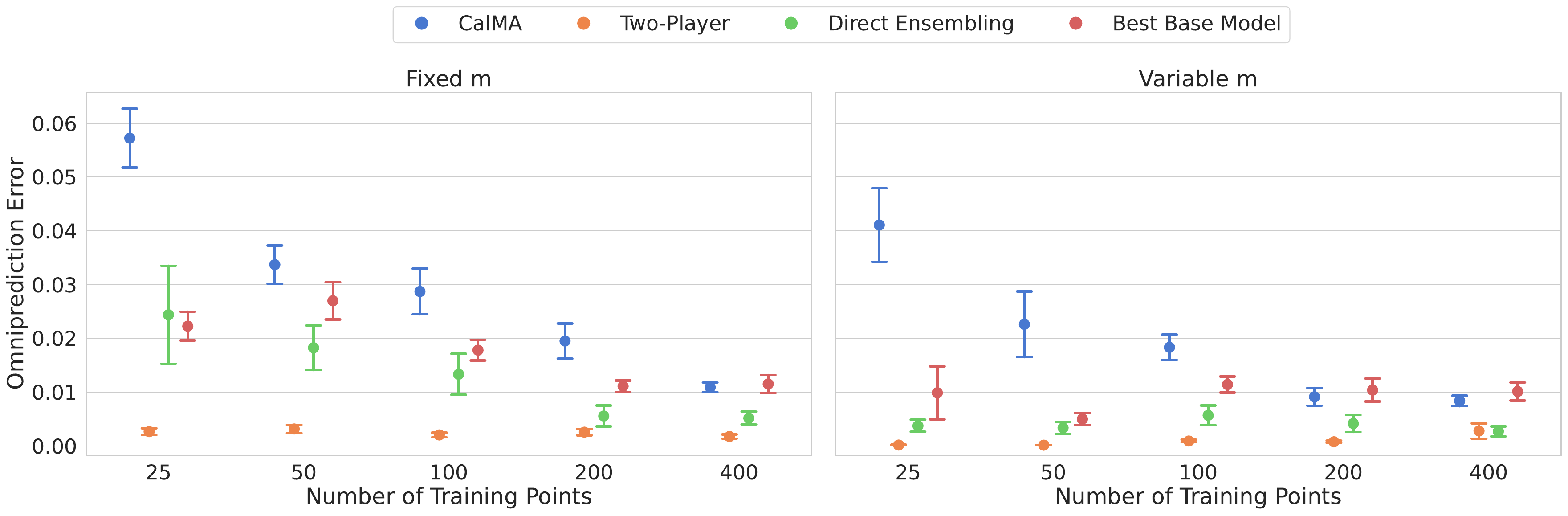}
    \caption{Realized average omniprediction error of the calibrated multiaccuracy (blue), two-player game based (orange), and  direct ensembling  (green) methods as well as the error of the best base model (red) across various sample sizes with $m=16$ fixed (left panel) or chosen variably as $m = 2^{\lfloor \log_2(\sqrt{n}) \rfloor}$ (right panel) on the M5 sales forecasting dataset. Dots and error bars show estimated means and standard errors obtained by evaluating the omniprediction error over 2000 test points for each of 20 draws of the training data. Hyperparameters for the calibrated multiaccuracy, two-player, and direct ensembling methods are set as $c = 0.5$, $c=32$, and $c=0$, respectively.}
    \label{fig:retail_example}
\end{figure}

Figure \ref{fig:retail_example} displays the results of this experiment over various sample sizes and settings of $m$. Similar to the previous section, we display the best performing hyperparameter for each method. Corresponding results for other parameter choices are shown in Figure \ref{fig:sales_var_hyper} in the appendix. In addition to the three omniprediction methods discussed above, this figure also shows results for the best performing base model, i.e., the predictor
\[
\hat{f} \in \underset{f \in \mathcal{F}}{\text{argmin}} \max_{i \in \{1,\dots,m\}} \frac{1}{n} \sum_{i=1}^n \ell_{\theta_i}(f_{\theta_i}(X_i),Y_i) - \frac{1}{n} \sum_{i=1}^n \ell_{\theta_i}(\hat{f}_{\theta_i}(X_i),Y_i),
\]
that minimizes the empirical omniprediction error on the day 14 data. 

As in the simulated example, the calibrated multiaccuracy method once again realizes the largest errors. Notably, this method is even outperformed by the best base model which offers no omniprediction guarantee. Once again, the two-player game based and direct ensembling approaches perform similarly when $m$ varies. When $m$ is fixed, the two-player game based method offers surprisingly strong performance, obtaining an omniprediction error of nearly zero for $n=25$. This is likely due to the fact that even before observing any training data the two-player game based approach forms an initial baseline ensemble of the available predictors (recall Lemma \ref{lem:min_max_opt_char}). In this example, this baseline performs well and thus the method does not require significant training data. On the other hand, the direct ensembling procedure requires additional training samples and only matches the two-player game based method for larger sample sizes.

\section{Discussion}\label{sec:other_targets}

This article studied three algorithmic frameworks for constructing predictors with low omniprediction error over the class of proper losses. Overall, our theoretical and empirical results show that methods based on calibrated multiaccuracy incur larger error rates than those based on two-player games and our direct ensembling approach. On the other hand, the latter two methods provide similar theoretical guarantees, with the two-player game based methods offering better empirical performance at small sample sizes.

\subsection{Extensions to other prediction targets}

In the previous sections we have chosen to focus on binary prediction problems in which the goal is to estimate the conditional probability function, $\mmp(Y=1 \mid X)$. Perhaps surprisingly, the algorithms and theory we have developed are not unique to this problem and can be extended to handle a large variety of estimation targets. To formalize this, let $T(P) \in \mmr$ denote a function that takes in a distribution $P$ on $\mathcal{Y}$ and returns the estimation target of interest. In the previous sections, we had $\mathcal{Y} = \{0,1\}$ and $T(P) = \mmp_P(Y = 1)$, but more generally one may consider common prediction tasks such as estimating the mean, $T(P) = \mme_P[Y]$ or $\tau$-quantile, $T(P) = \inf\{t : \mmp_P(Y \leq t) \geq \tau\}$ with $\mathcal{Y} = \mmr$. We say that $T$ is an elicitable property of $P$ if there exists at least one loss function which is minimized at $T(P)$, i.e., there exists $\ell$ such that for all $P$,
\[
T(P) \in \text{argmin}_t \mme_P[\ell(t,Y)].
\]
It is worth noting that while common prediction targets such as means and quantiles are elicitable, not every property of a distribution can be obtained this way. A notable example is the conditional value-at-risk which is well-known to be non-elicitable \citep{Gneiting2011}. 

Now, restricting to elicitable properties, the goal is to design predictors that estimate $T(P_{Y\mid X})$ well under all possible losses for $T$. As above, we say that $\ell$ a proper loss for $T$ if $T(P) \in \text{argmin}_t \mme_P[\ell(t,Y)]$ for all $P$ and strictly proper if $T(P)$ is the unique minimizer. The key technical tool that allowed us to handle arbitrary proper losses in binary prediction was Theorem \ref{thm:proper_loss_decomp}, which gave a decomposition of proper losses as mixtures of a one-dimensional family of weighted 0-1 losses. To extend our results beyond binary prediction, we will leverage the following result of \citet{Steinwart2014}, which demonstrates the existence of similar decompositions for other targets. This result introduces the technical requirement that $T$ is strictly locally non-constant. Informally, this means that slight changes in $P$ can shift $T(P)$ up or down. A more precise definition of this property is given as Definition 4 in \citet{Steinwart2014}.

\begin{proposition}[Variant of Corollary 9 of \citet{Steinwart2014}]\label{prop:general_loss_decomp}
    Let $(\mathcal{Y}, \mathcal{A}, \mu)$ be a separable, finite measure space, $\mathcal{P}$ be a set of $\mu$-absolutely continuous distributions on $\mathcal{Y}$ and
     $T : \mathcal{P} \to \mmr$ be continuous, elicitable, strictly locally non-constant, and such that $\textup{Image}(T)$ is an interval. Then, there exists a measurable function $V: \textup{Image}(T) \times \mathcal{Y} \to \mmr$ that identifies $T$, i.e., a function $V$ with the property that for all $t \in \textup{Interior}(\textup{Image}(T))$,
    \[
    \mme_{Y \sim P}[V(t,Y)] = 0 \iff t = T(P) \ \ \  \text{ and } \ \ \ \mme_{Y \sim P}[V(t,Y)] > 0 \iff t > T(P).
    \]
    Moreover, all strictly proper losses $\ell$ for $T$ that are locally-Lipschitz in their first argument can be decomposed as
   \begin{equation}\label{eq:general_proper_loss_decomp}
    \ell(t,y) = \int_{-\infty}^{\infty} V(\theta,y) \bone\{t \leq \theta\} w(\theta) d\theta + \kappa(y),\ \textup{for all $t \in \mmr$ and $\mu$-almost all $y \in \mathcal{Y}$.}
    \end{equation}
Here, $\kappa : \mathcal{Y} \to \mmr$ and $w : \mmr \to [0,\infty)$ are functions that depend on $\ell$.    \
\end{proposition}

A key component of Proposition \ref{prop:general_loss_decomp} is the identification function, $V$. Common examples include $V(t,y) = t-y$, which identifies mean, and $V(t,y) =  \bone\{y \leq t\} -\tau$, which identifies the $\tau$ quantile. The perhaps surprising insight of this proposition is that any (appropriately smooth) proper loss for the mean or $\tau$ quantile can be written as a mixture over these identification functions.

With Proposition \ref{prop:general_loss_decomp} in hand, algorithms for other point prediction targets can be obtained directly by replacing the weighted 0-1 losses appearing in our methods with the threshold loss $\ell^T_{\theta}(t,y) := V(\theta,y)\bone\{t \leq \theta\}$. In particular, the decomposition given in (\ref{eq:general_proper_loss_decomp}) is essentially identical to the decomposition for binary prediction that we gave previously in Theorem \ref{thm:proper_loss_decomp}. Moreover, similar to the binary case, the loss $\ell^T_{\theta}(t,y)$ is proper and can be interpreted as evaluating whether $T$ falls above or below $\theta$. By replacing all instances of $\ell_{\theta}$ with $\ell_{\theta}^T$ in the previous sections, one may adapt Algorithms \ref{alg:online_omni} and \ref{alg:ensemb_main} to construct predictors $\hat{t}(\cdot)$ satisfying the corresponding omniprediction guarantee
\[
\sup_{\ell^T, f \in \mathcal{F}} \mme_{(X,Y)}[\ell^T(\hat{t}(X),Y) ] - \mme_{(X,Y)}[\ell^T(f(X),Y)] \leq \tilde{O}_{\mmp}\left( \sqrt{\frac{\textup{VC}(\mathcal{F})}{n}} \right),
\]
where the supremum is over all proper losses for $T$ satisfying appropriate regularity conditions. Making this statement precise requires some minor additional technical assumptions to ensure that the weight function, $w(\theta)$ is appropriately bounded and the parameters, $\theta$ can be discretized. As this is not the main focus of this work, we do not pursue this here.

A more challenging task is to extend our results beyond point prediction problems. For instance, given a multiclass outcome $Y \in \{1,\dots,k\}$ we may attempt to derive estimates of the entire vector of conditional probabilities $(\mmp(Y = 1 \mid X),\dots, \mmp(Y=k \mid X))$. Unfortunately, characterizing the class of proper losses in this instance is significantly more challenging. While previously we could decompose proper losses in terms of a one-dimensional family,   \citet{Kleinberg2023} shows that the space of multiclass proper losses is fundamentally more complex and it is impossible to construct a finite dimensional class of loss functions that admit a similar decomposition. Determining whether efficient omniprediction algorithms exist in this setting is an interesting open problem for future work.

\section*{Acknowledgments}

This work was supported by the Office of Naval Research, ONR grant N00014-20-1-2787. The authors thank Sivaraman Balakrishnan for helpful discussions.

\newpage

\bibliographystyle{plainnat}
\bibliography{omni_references}

\newpage 

\appendix

\section{Proofs for Section \ref{sec:non-param}}

In this section we prove Proposition \ref{prop:non_param_connection}.

\begin{proof}[Proof of Proposition \ref{prop:non_param_connection}]
To get the upper bound, fix any bounded, proper loss $\ell \in \mathcal{L}_0$. Then,
\begin{align*}
\mme[\ell(p(X),Y)] - \mme[\ell(p^*(X),Y)] & = \mme[\ell(p(X),Y) - \ell(p^*(X),Y)] - \mme_{X,Y' \mid X \sim p(X)}[\ell(p(X),Y') - \ell(p^*(X),Y')]\\
& \ \ + \mme_{X,Y' \mid X \sim p(X)}[\ell(p(X),Y') - \ell(p^*(X),Y')]\\
 & \leq \mme[\ell(p(X),Y) - \ell(p^*(X),Y)] - \mme_{X,Y' \mid X \sim p(X)}[\ell(p(X),Y') - \ell(p^*(X),Y')] \\
 & = \mme_X[(p^*(X) -p(X))(\ell(p(X),1) - \ell(p(X),0) - \ell(p^*(X),1) + \ell(p^*(X),0))]\\
 & \leq 2 \mme_X[|p(X) -p^*(X)|],
\end{align*}
where the first inequality uses the fact that $\ell$ is proper to bound the second term by $0$.

For the lower bound, let $m \in \mmn$ be an positive integer to be specified shortly. Then,
\begin{align*}
&  \mme[|p(X) - p^*(X)|]  = 2\mme\left[\left| p^*(X) -   \frac{p(X) + p^*(X)}{2} \right| \right]\\
 & \leq \frac{2}{m}  + 2\mme\left[\left| p^*(X) -   \frac{p(X) + p^*(X)}{2} \right| \bone\left\{ |p^*(X) - p(X)| > \frac{2}{m} \right\}\right]\\
& = \frac{2}{m} + \sum_{i=0}^m 2\mme\left[ \left| p^*(X) -   \frac{p(X) + p^*(X)}{2}  \right|   \bone\left\{ |p^*(X) - p(X)| > \frac{2}{m} \right\} \bone\left\{ \left\lfloor m\frac{p(X) + p^*(X)}{2} \right\rfloor = i \right\} \right]\\
& \leq \frac{4}{m} + \sum_{i=0}^m 2\mme\left[ \left| p^*(X) -  \frac{i}{m} \right|   \bone\left\{ |p^*(X) - p(X)| > \frac{2}{m} \right\} \bone\left\{ \left\lfloor m\frac{p(X) + p^*(X)}{2} \right\rfloor = i \right\} \right] \\
& \leq \frac{4}{m} + \sum_{i=0}^m 2\mme\left[ \left| p^*(X) -  \frac{i}{m} \right|  \bone\left \{p(X) \leq \frac{i}{m} < p^*(X) \text{ or } p^*(X) \leq \frac{i}{m} < p(X) \right\} \right]\\
& = \frac{4}{m} + \sum_{i=0}^m 2(\mme[\ell_{i/m}(p(X),Y)] - \mme[\ell_{i/m}(p^*(X),Y)]),
\end{align*}
where we recall that $\ell_{i/m}$ denotes the proper loss function given by
\[
\ell_{i/m}(p,y) = \frac{i}{m} \bone\left\{p > \frac{i}{m}, y = 0\right\} + \left(1- \frac{i}{m}\right) \bone\left\{p \leq \frac{i}{m}, y = 1\right\}.
\] 
So, rearranging we find that
\[
\sup_{\ell \in \mathcal{L}_0} \mme[\ell(p(X),Y)] - \mme[\ell(p^*(X),Y)] \geq \frac{\mme[|p(X) - p^*(X)|] }{2(m+1)} - \frac{4}{2m(m+1)}. 
\]
Finally, setting $m = \lfloor 7\mme[|p(X) - p^*(X)|]^{-1} \rfloor - 1$ gives 
\begin{align*}
 & \frac{\mme[|p(X) - p^*(X)|] }{m+1} - \frac{4}{m(m+1)}\\
 & \geq \frac{\mme[|p(X) - p^*(X)|]^2}{7} - \frac{4}{(7\mme[|p(X) - p^*(X)|]^{-1} - 2)(7\mme[|p(X) - p^*(X)|]^{-1}  - 1)}\\
 & \geq \frac{\mme[|p(X) - p^*(X)|]^2}{7} - \frac{4\mme[|p(X) - p^*(X)|]^2}{30}\\
 & = \frac{\mme[|p(X) - p^*(X)|]^2}{105},
\end{align*}
where to get the second inequality we have used the fact that $\mme[|p(X) - p^*(X)|] \leq 1$.    
\end{proof}

\section{Proofs for Section \ref{sec:cma_omni}}\label{sec:app_ca_ma}

In this section we prove Propositions \ref{prop:ca_ma_lower_bound}, \ref{prop:ma_lower_bound} and \ref{prop:ca_upper_bound}, which give lower and upper bounds on the minimax rate of calibrated multiaccuracy. We begin by proving the lower bound for calibrated multiaccuracy appearing in Proposition \ref{prop:ca_ma_lower_bound}.

\begin{proof}[Proof of Proposition \ref{prop:ca_ma_lower_bound}]

We will prove this result using Fano's method \citep{Yu1997}. Let $k \in \mmn$ be a large value that we will specify shortly and set $X_i$ to be uniformly distributed on $\{\frac{1}{k},\frac{2}{k},\dots,1\}$. By the Varshamov–Gilbert lemma (see, e.g., Lemma 4.7 of \citet{Massart2007}) we may find a collection of vectors $V \subseteq \{0,1\}^k$ such that $|V| \geq \exp(k/4)$ and for all $v, v' \in V$ with $v \neq v'$, $\|v - v'\|_1 \geq k/8$. Our goal will be to apply Fano's inequality to the set of distributions given by $p^*(X) = p_v(X) = \frac{1}{4} + \frac{X}{2} + \delta v_X$ for $v \in V$ and some appropriately small value $\delta>0$. The idea here is that in order to be multiaccurate the predictor $\hat{p}(X)$ must correctly capture the linear component of $p_v(X)$ given by the term $\frac{X}{2}$. Then, the only way for $\hat{p}(X)$ to additionally be calibrated is if it accurately determines the value of $v_x$ for most values of $x \in \{\frac{1}{k},\frac{2}{k},\dots,1\}$. This latter problem is difficult and suffers a worst-case estimation rate of $\Omega(n^{-2/5})$.

To formalize this, we begin by lower bounding the ability of the predictor to hedge between two sign vectors. In particular, fix $v, v' \in V$ with $v \neq v'$. Then, we will lower bound
\[
\inf_{p} \max_{p^* \in \{p_v, p_{v'}\}}  \max\{\mme_{p^*}[X(Y - p(X))], \mme[|p(X) - \mme[p^*(X) \mid p(X)]|]\},
\]
where the infimum is taken over all functions $p : \{\frac{1}{k},\frac{2}{k},\dots,1\} \to [0,1]$ and the notation $\mme_{p^*}$ is used to denote the distribution in which $X \sim \text{Unif}(\{\frac{1}{k},\frac{2}{k},\dots,1\})$ and $Y \mid X \sim \text{Ber}(p^*(X))$.

Fix any $p : \{\frac{1}{k},\frac{2}{k},\dots,1\} \to [0,1]$. Let $p_1,\dots,p_r$ denote the distinct values in the support of $p(X)$ and for $i \in \{1,\dots,r\}$ let $G_i = \{x \in \{\frac{1}{k},\frac{2}{k},\dots,1\} : p(x) = p_i\}$. For ease of notation, define the maximum calibration error as
\[
 \text{ECE}_{\textup{max}}(p;v,v')  = \max_{p^* \in \{p_v, p_{v'}\}} \left| \mme[|p(X) - \mme[p^*(X) \mid p(X)]|] \right| = \max_{\tilde{v} \in \{v,v'\}} \sum_{i=1}^r \frac{|G_i|}{k} \left| \frac{1}{|G_i|} \sum_{x \in G_i} \frac{1}{4} + \frac{x}{2} + \delta \tilde{v}_x - p_i \right|,
\]
and note that
\begin{align*}
\sum_{i=1}^r \frac{|G_i|}{k}  \left| \frac{1}{|G_i|}\sum_{x \in G_i}  (v_x - v'_x )\right| & \leq   \frac{1}{\delta} \sum_{i=1}^r \frac{|G_i|}{k} \left( \left| \frac{1}{|G_i|}\sum_{x \in G_i} \frac{1}{4} + \frac{x}{2} + \delta v_x - p_i \right| + \left| \frac{1}{|G_i|}\sum_{x \in G_i} \frac{1}{4} + \frac{x}{2} + \delta v'_x - p_i \right| \right) \\
& \leq \frac{2\text{ECE}_{\textup{max}}(p;v,v') }{\delta}.
\end{align*}
In particular, applying this bound alongside our assumptions on $V$ gives
\begin{align*}
     \frac{k}{8} \leq \|v-v'\|_1 & \leq \sum_{i=1}^r \bone\{G_i = 1\} |v_i - v'_i|  + \sum_{i=1}^r|G_i| \bone\{G_i > 1\}\\
    & \leq \sum_{i=1}^r |G_i|  \left| \frac{1}{|G_i|}\sum_{x \in G_i}  (v_x - v'_x )\right| + \sum_{i=1}^r |G_i|\bone\{G_i > 1\}\\
    & \leq  \frac{2\text{ECE}_{\textup{max}}(p;v,v') }{\delta} + \sum_{i=1}^r |G_i|\bone\{G_i > 1\},
\end{align*}
and rearranging we have that 
\[
\sum_{i=1}^r|G_i| \bone\{G_i > 1\} \geq \frac{k}{8} - \frac{2\text{ECE}_{\textup{max}}(p;v,v') }{\delta}.
\]
On the other hand, by considering the multiaccuracy error with $g(x) = x$ we find that 
\begin{align*}
& \mme_{p_v}\left[X(Y - p(X))   \right] \\
& \geq  \sum_{i=1}^r \frac{|G_i|}{k} \left( \frac{1}{|G_i|} \sum_{x \in G_i} x \left(\frac{1}{4} + \frac{x}{2} + \delta v_x \right) - \frac{1}{|G_i|}\sum_{x \in G_i} x \left( \frac{1}{4} + \frac{1}{|G_i|}\sum_{x \in G_i} \frac{x}{2} + \delta v_x \right) \right) - \text{ECE}_{\textup{max}}(p;v,v')  \\
& \geq\sum_{i=1}^r \frac{|G_i|}{2k}  \left(  \frac{1}{|G_i|} \sum_{x \in G_i} x^2 -  \left(\frac{1}{|G_i|} \sum_{x \in G_i} x \right)^2 \right)   - \delta - \text{ECE}_{\textup{max}}(p;v,v')\\
& = \sum_{i=1}^r \frac{|G_i|}{4k}   \frac{1}{|G_i|^2} \sum_{x, x' \in G_i} \left( x-x' \right)^2  - \delta - \text{ECE}_{\textup{max}}(p;v,v') \\
& \geq \sum_{i=1}^r \frac{|G_i|}{4k}   \left(1-\frac{1}{|G_i|}\right) \frac{1}{k^2} - \delta - \text{ECE}_{\textup{max}}(p;v,v') \\
& \geq \frac{1}{8k^3} \sum_{i=1}^r |G_i|   \bone\{|G_i| > 1\} - \delta - \text{ECE}_{\textup{max}}(p;v,v')\\
& \geq \frac{1}{8k^3}\left(\frac{k}{8} - \frac{2\text{ECE}_{\textup{max}}(p;v,v') }{\delta}\right) - \delta - \text{ECE}_{\textup{max}}(p;v,v'),
\end{align*}
and rearranging the first and last inequalities gives
\[
\mme_{ p_v}\left[X(Y - p(X))   \right] + \text{ECE}_{\textup{max}}(p;v,v') + \frac{1}{4k^3}\frac{\text{ECE}_{\textup{max}}(p;v,v') }{\delta} \geq \frac{1}{64k^2} - \delta.
\]
Finally, setting $\delta = \frac{1}{128 k^2}$ we find that 
\[
\inf_p \max_{p^* \in \{p_v, p_{v'}\}}  \max\{\mme_{p^*}[X(Y - p(X))], \mme[|p(X) - \mme[p^*(X) \mid p(X)]|]\} \geq \frac{k}{k + 32} \frac{1}{128k^2}.
\]

With this inequality in hand, the proof of our desired result now follows from the following straightforward application of Fano's inequality (e.g, Lemma 3 of \citet{Yu1997}). Let $\hat{p} : \mathcal{X} \to [0,1]$ denote any estimator. Define an associated classifier by 
\[
\hat{v} \in \underset{v \in V}{\text{argmin}}   \max\{|\mme_{p_v}[X(Y - \hat{p}(X)) ]|, \mme[|p(X) - \mme[p_v(X) \mid \hat{p}(X)]|]\},
\]
where both here and in what follows the expectations are taken with respect to $(X,Y)$ with the estimator $\hat{p}(\cdot)$ (which is a random function of the training data) held fixed.
By our previous calculations, we have that for any $v^* \in V$, 
\[
  \max\left\{ \mme_{p_{v^*}}[X(Y - \hat{p}(X))], \mme[|p(X) - \mme[p_{v^*}(X) \mid \hat{p}(X)]|]\right\} \geq \frac{k}{k + 32} \frac{1}{128k^2} \bone\{\hat{v} \neq v^*\},
\]
and thus,
\begin{align*}
& \sup_{v^* \in V} \mme_{\{(X_i,Y_i)\}_{i=1}^n \stackrel{i.i.d.}{\sim} p_{v^*}}\left[ \max\left\{\mme_{p_{v^*}}[X(Y - \hat{p}(X))], \mme[|\hat{p}(X) - \mme[p_{v^*}(X) \mid \hat{p}(X)]|]\right\} \right]\\
& \geq \mme_{v^* \sim \text{Unif}(V)}\left[ \mme_{\{(X_i,Y_i)\}_{i=1}^n \stackrel{i.i.d.}{\sim} p_{v^*}}\left[ \max\left\{ \mme_{p_{v^*}}[X(Y - \hat{p}(X))], \mme[|\hat{p}(X) - \mme[p_{v^*}(X) \mid \hat{p}(X)]|]\right\} \right] \right]\\
& \geq \frac{k}{k + 16} \frac{1}{128k^2} \mmp_{v^* \sim \text{Unif}(V), \{(X_i,Y_i)\}_{i=1}^n \stackrel{i.i.d.}{\sim} p_v }(\hat{v} \neq v^*)\\
& \geq \frac{k}{k + 16} \frac{1}{128k^2} \left(1 - \frac{\frac{1}{|V|^2} \sum_{v,v' \in V} n D_{KL}(p_v || p_{v'})  + \log(2)}{\log(|V|)} \right)\\
\end{align*}
where $D_{KL}(p_v || p_{v'}) $ denotes the KL-divergence between the distribution of $(X,Y)$ under $p_v$ and $p_{v'}$. By a direct calculation,
\begin{align*}
D_{KL}(p_v || p_{v'}) & = \mme_X\left[p_v(X) \log\left(\frac{p_v(X)}{p_{v'}(X)} \right) + (1-p_v(X)) \log\left(\frac{1-p_v(X)}{1-p_{v'}(X)} \right)\right]\\
& \leq \mme_X\left[p_v(X)\left(\frac{p_v(X)}{p_{v'}(X)}  - 1 \right)  +  (1-p_v(X))\left(\frac{1-p_v(X)}{1-p_{v'}(X)}  - 1\right) \right]\\
& = \mme_X\left[\frac{(p_v(X) - p_{v'})^2}{p_{v'}(X)(1-p_{v'}(X))}\right]\\
& \leq \frac{64}{7} \delta^2,
\end{align*}
where the last inequality holds for $\delta \leq 1/8$. Plugging this into the previous expression gives a lower bound of 
\[
 \frac{k}{k + 32} \frac{1}{128k^2}\left(1-  \frac{n\frac{64}{7} \delta^2 + \log(2)}{k/4}\right) = \frac{k}{k + 32} \frac{1}{128k^2}\left(1-  \frac{n\frac{64}{7} 128^{-2}k^{-4} + \log(2)}{k/4}\right).
\]
The desired result follows immediately by taking $k = C\lceil n^{-1/5}\rceil$ for an appropriately chosen constant $C$.

\end{proof}

We next give a proof of our lower bound for multiaccuracy given in Proposition \ref{prop:ma_lower_bound}.

\begin{proof}[Proof of Proposition \ref{prop:ma_lower_bound}]
    For ease of notation, let $d:=\textup{VC}(\mathcal{G})$.
    We once again proceed using Fano's Method. By definition of the VC dimension, we may find a set of points $x_1,\dots,x_d$ such that for all $v \in \{-1,1\}^d$ there exists $g_v \in \mathcal{G}$ with $g_v(x_i) = v_i$ for all $i \in \{1,\dots,d\}$. Let $V$ be as in the proof of Proposition \ref{prop:ca_ma_lower_bound} and consider the set of distributions given by $X \sim \text{Unif}(x_1,\dots,x_d)$ and $Y \mid X \sim \text{Ber}(\frac{1+\delta g_v(X)}{2})$ for some small value $\delta > 0$ that we will specify shortly. Let $\mme_{v}$ denote the expectation over this distribution on $(X,Y)$. For any $v \neq v'$ with $v,v' \in V$ and $p : \{x_1,\dots,x_d\} \to [0,1]$ we have that 
    \begin{align*}
    & \max_{v^* \in \{v,v'\}} \sup_{g \in \mathcal{G}}\mme_{v^*}[g(X)(Y-p(X))]  = \max_{v^* \in \{v,v'\}} \sup_{g \in \mathcal{G}}\mme_X\left[g(X)\left(\frac{1+\delta g_{v^*}(X)}{2}-p(X)\right)\right]\\
    & \geq \frac{1}{2}   \sup_{g \in \mathcal{G}}\mme_X\left[g(X)\left(\frac{1+\delta g_{v}(X)}{2}-p(X)\right)\right] + \frac{1}{2} \sup_{g \in \mathcal{G}}\mme_X\left[g(X)\left(\frac{1+\delta g_{v'}(X)}{2}-p(X)\right)\right]\\
    & \geq \frac{1}{2}   \sup_{g \in \mathcal{G}} \mme_X\left[g(X)\left(\frac{1+\delta g_{v}(X)}{2}-\frac{1+\delta g_{v'}(X)}{2}\right)\right]\\
    & = \frac{\delta}{4} \mme_X[|g_v(X) - g_{v'}(X)|] = \frac{1}{4} \delta \frac{\|v - v'\|_1}{d} \geq \frac{\delta }{32}.
    \end{align*}
    So, proceeding exactly as in the proof of Proposition \ref{prop:ca_ma_lower_bound}, we obtain the lower bound,
    \[
    \min_{\hat{p}} \sup_{P_{XY}} \sup_{g \in \mathcal{G}} \mme\left[g(X)(Y-\hat{p}(X)) \right] \geq \frac{\delta }{32} \left(1- \frac{n \frac{64}{7}\delta^2 + \log(2)}{d\log(2)} \right).
    \]
    Setting $\delta = C\sqrt{d/n}$ for a sufficiently small constant $C>0$ gives the result.
\end{proof}

We now turn to the proof of Proposition \ref{prop:ca_upper_bound}. Our algorithm for obtaining calibrated multiaccuracy will follow a similar structure to the two-player game based algorithms for omniprediction introduced in Section \ref{sec:online_omni}. Namely, we expand the calibration and multiaccuracy criteria as a set of objectives and use a multiplicative weights algorithm to obtain useful mixtures of these targets.

To state this method formally, fix a hyperparameter $m \in \mmn$. Our goal will be to learn a predictor that returns randomized outputs in $\{\frac{1}{m},\frac{2}{m},\dots,1\}$. Let $\mathcal{G}_m := \{g : \{\frac{1}{m},\frac{2}{m},\dots,1\} \to \{-1,1\}\}$ denote the set of sign functions on $\{\frac{1}{m},\frac{2}{m},\dots,1\}$.  Let $\Delta_m$ denote the space of probability distributions on $\{\frac{1}{m},\frac{2}{m},\dots,1\}$ and note that for any randomized predictor $P : \mathcal{X} \to \Delta_m$ the expected calibration error can be written as
\[
\mme_{(X,Y), p \mid X \sim P(X)}[|p - \mme[Y \mid p]|] = \sup_{g \in \mathcal{G}_m}\mme_{(X,Y), p \mid X \sim P(X)}[g(p)(Y - p)].
\]
So, to guarantee calibration it is sufficient to guarantee that our predictor gives multiaccurate predictions with respect to each $g \in \mathcal{G}_m$. Combining this with the original multiaccuracy targets specified by $\mathcal{G}$ gives us the necessary set of objectives for a two-player game based algorithm. Formal statement of this method is given in Algorithm \ref{alg:ca_ma}. As stated in Proposition \ref{prop:ca_ma_bound}, this algorithm achieves calibrated multiaccuracy at a rate of $O_{\mmp}(\sqrt{\log(|\mathcal{G}|)/n} + n^{-1/3})$. This proves Proposition \ref{prop:ca_upper_bound}.

\begin{algorithm}
\KwData{Data $\{(X_i,Y_i)\}_{i=1}^n$, finite function class $\mathcal{G}$, hyperparameters $m \in \mmn$, $\eta > 0$.}
$\mathcal{G}_{\pm} = \mathcal{G} \cup \{-g : g \in \mathcal{G}\}$\;
$q_{g}(1) = \frac{1}{|\mathcal{G}_{\pm} \cup \mathcal{G}_m|}$, for all $g \in \mathcal{G}_{\pm}  \cup \mathcal{G}_m$\;
\For{$i = 1,\dots,n$ }{
    $\hat{P}_i(X) = \min_{P \in \Delta_{m}} \max_{p_y \in [0,1]} \sum_{g \in \mathcal{G}_{\pm}} q_g(i) \mme_{p \sim P}[g(X)(p_y - p)] +  \sum_{g \in \mathcal{G}_m} q_g(i) \mme_{p \sim P}[g(p)(p_y - p)]$\;
    $\tilde{q}_g(i+1) = \tilde{q}_g(i) \exp(\eta \mme_{p \sim \hat{P}_i(X_i)}[g(X_i)(Y_i - p)])$, $\forall g \in \mathcal{G}_{\pm}$\;
    $\tilde{q}_g(i+1) = \tilde{q}_g(i) \exp(\eta \mme_{p \sim \hat{P}_i(X_i)}[g(p)(Y_i - p)])$, $\forall g \in \mathcal{G}_m$\;
    $q_g(i+1) = \frac{\tilde{q}_g(i+1)}{\sum_{g' \in \mathcal{G}_{\pm} \cup \mathcal{G}_m} \tilde{q}_{g'}(i+1)}$ , $\forall g \in \mathcal{G}_{\pm} \cup \mathcal{G}_m$\;
}
    \Return $\hat{P} = \frac{1}{n}\sum_{i=1}^n \hat{P}_i $
\caption{Calibrated Multiaccuracy}
\label{alg:ca_ma}
 \end{algorithm}

 \begin{proposition}\label{prop:ca_ma_bound}
     Let $\hat{P}$ denote the randomized predictor returned by Algorithm \ref{alg:ca_ma} with hyperparameters $\eta = \sqrt{(\log(|\mathcal{G}|) + m)/n}$ and $m = \lceil n^{1/3} \rceil$. Then,
     \[
     \max \left\{ \sup_{g \in \mathcal{G}} \left|\mme_{(X,Y),p \sim \hat{P}(X)}\left[g(X)(Y - p)\right] \right|, \mme_{(X,Y),p \sim \hat{P}(X)}\left[\left|p - \mme[Y\mid p]\right|\right] \right\} \leq O_{\mmp}\left( \sqrt{\frac{\log(|\mathcal{G}|)}{n}} + \frac{1}{n^{1/3}} \right).
     \]
 \end{proposition}

\begin{proof}
We first show that $\hat{P}$ is multiaccurate. Fix any $g \in \mathcal{G}$. By definition, we have that 
\[
\mme_{(X,Y), p \sim \hat{P}(X)}[g(X)(Y - p)] = \frac{1}{n} \sum_{i=1}^n \mme_{X,Y, p \sim \hat{P}_i(X)}[g(X)(Y - p)].
\]
Now, by the Azuma-Hoeffding inequality (Theorem \ref{thm:AH} below) we may guarantee that for any $c>0$,
\begin{align*}
& \mmp\left( \sup_{g \in \mathcal{G}} \left| \frac{1}{n} \sum_{i=1}^n \mme_{(X,Y), p \sim \hat{P}_i(X)}[g(X)(Y - p)] - \frac{1}{n} \sum_{i=1}^n \mme_{p \sim \hat{P}_i(X_i)}[g(X_i)(Y_i - p)] \right| \geq c\sqrt{\frac{\log(|\mathcal{G}|)}{n}} \right)\\ 
&  \leq  2\exp\left(-\frac{c^2}{8}\right).
\end{align*}
Applying this to the previous expression, we find that 
\begin{align*}
\mme_{(X,Y), p \sim \hat{P}(X)}[g(X)(Y - p)] \leq \frac{1}{n} \sum_{i=1}^n \mme_{p \sim \hat{P}_i(X_i)}[g(X_i)(Y_i - p)] + O_{\mmp}\left( \sqrt{\frac{\log(|\mathcal{G}|)}{n}}\right).
\end{align*}
The updates for $q_g$ given in Algorithm \ref{alg:ca_ma} are exactly the updates for the well-known Hedge method \citep{Vovk1990, Littlestone1994, Freund1997}. By standard regret bounds for this algorithm (see Theorem \ref{thm:hedge} below), we have the inequality 
\begin{align*}
\frac{1}{n} \sum_{i=1}^n \mme_{p \sim \hat{P}_i}[g(X_i)(Y_i - p)] \leq \frac{1}{n} & \sum_{i=1}^n \sum_{g' \in \mathcal{G}_{\pm}} q_{g'}(i) \mme_{p \sim \hat{P}_i(X_i)}[g'(X_i)(Y_i - p)]\\
& + \frac{1}{n} \sum_{i=1}^n\sum_{g' \in \mathcal{G}_{m}} q_{g'}(i) \mme_{p \sim \hat{P}_i(X_i)}[g'(p)(Y_i - p)] + O\left(\sqrt{\frac{\log(|\mathcal{G}|) + m}{n}} \right).
\end{align*}
Finally, by definition of $\hat{P}_i(X_i)$ and von Neumann's minimax theorem \citep{vonNeumann1944},
\begin{align*}
& \sum_{g' \in \mathcal{G}_{\pm}} q_{g'}(i) \mme_{p \sim \hat{P}_i(X_i)}[g'(X_i)(Y_i - p)] + \sum_{g' \in \mathcal{G}_{m}} q_{g'}(i) \mme_{p \sim \hat{P}_i(X_i)}[g'(p)(Y_i - p)] \\
& \leq  \min_{P \in \Delta_m} \sup_{p_y \in [0,1]} \sum_{g' \in \mathcal{G}_{\pm}} q_{g'}(i) \mme_{p \sim P}[g'(X_i)(p_y - p)] + \sum_{g' \in \mathcal{G}_{m}} q_{g'}(i) \mme_{p \sim P}[g'(p)(p_y - p)]\\
& =  \sup_{p_y \in [0,1]} \min_{P \in \Delta_m}  \sum_{g' \in \mathcal{G}_{\pm}} q_{g'}(i) \mme_{p \sim P}[g'(X_i)(p_y - p)] + \sum_{g' \in \mathcal{G}_{m}} q_{g'}(i) \mme_{p \sim P}[g'(p)(p_y - p)]\\
& \leq \frac{1}{m},
\end{align*}
where to get the last inequality one may simply set $P$ to give probability one to the element of $\{\frac{1}{m},\frac{2}{m},\dots,1\}$ that is closest to $p_y$.

Combining all of the previous steps, we arrive at the final bound
\[
\sup_{g \in \mathcal{G}}  \mme_{(X,Y), p \sim \hat{P}(X)}[g(X)(Y - p)] \leq O_{\mmp}\left( \sqrt{\frac{\log(|\mathcal{G}|)}{n}}\right) + O\left(\sqrt{\frac{\log(|\mathcal{G}|) + m}{n}} \right) + \frac{1}{m} = O_{\mmp}\left( \sqrt{\frac{\log(|\mathcal{G}|)}{n}} + \frac{1}{n^{1/3}}\right) ,
\]
by our choice of $m = \lceil n^{1/3} \rceil$. A bound on the multiaccuracy follows by applying the same argument to $-g$. 

Finally, to bound the expected calibration error we simply note that since $\hat{P}$ is supported on $\{\frac{1}{m}, \frac{2}{m},\dots, 1\}$
\[
\mme_{(X,Y),p   \sim \hat{P}(X)}[\left|p - \mme[Y \mid p] \right|] =\mme_{(X,Y),p  \sim \hat{P}(X)}[ \text{sign}(p - \mme[Y \mid p])(p-Y)] = \sup_{g \in \mathcal{G}_m} \mme_{(X,Y),p   \sim \hat{P}(X)}[g(p)(p - Y)]. 
\]
This final quantity can be bounded by following the preceding argument for the bound on the multiaccuracy.

\end{proof}

\section{Extensions of Theorem \ref{thm:proper_loss_decomp} beyond left-continuity}\label{sec:app_left_cont_relaxation}

While we will not pursue this in detail, it is possible to extend Theorem \ref{thm:proper_loss_decomp} beyond left-continuous losses. To motivate this, let us first consider the discontinuity point of $\ell_{\theta}$. From (\ref{eq:ell_theta_proper}), we see that when the true underlying probability is equal to $\theta$ all predictions have the same expected loss. As a result, one can modify the value of the loss substantially at $p=\theta$ without affecting its propriety. Indeed, with some additional calculation one can verify that the family of losses 
\[
\ell_{\theta,\beta} = \begin{cases}
    \theta, \text{ if } p > \theta \text{ and } y = 0,\\
    1-\theta, \text{ if } p < \theta \text{ and } y = 1,\\
    \theta(1-\theta) + \beta(y-\theta), \text{ if } p = \theta,
\end{cases}
\]
is proper for all $\theta \in [0,1]$ and $\beta \in [-\theta,1-\theta]$. By varying the parameter $\beta$, one can encode a variety of different jump discontinuities in $\ell_{\theta,\beta}$. While not a complete proof, the calculations in \citet{Kleinberg2023} suggest that these jumps are in fact sufficient to capture all possible discontinuities in proper losses and, in particular, to extend Theorem \ref{thm:proper_loss_decomp} to a decomposition of arbitrary proper losses in terms of mixtures over the two-parameter class $\{\ell_{\theta,\beta} : \theta \in [0,1], \beta \in [-\theta,1-\theta]\}$. As discussed in the main text, we do not believe that this extra layer of complexity has a large impact on practical results for omniprediction and thus we have chosen to omit these details and restrict ourselves to left-continuous losses.

\section{Proofs for Section \ref{sec:simplification}}

In this section we prove Lemma \ref{lemma:disc_bound}.

\begin{proof}[Proof of Lemma \ref{lemma:disc_bound}]
    Fix any $\theta \in [0,1]$ and  $\epsilon > 0$. Let $f_{\theta,\epsilon}$ be such that
    \[
    \sup_{f \in \mathcal{F}} \mme[\ell_{\theta}(p(X),Y)] - \mme[\ell_{\theta}(f(X),Y)] \leq \mme[\ell_{\theta}(p(X),Y) - \mme[\ell_{\theta}(f_{\theta,\epsilon}(X),Y)] + \epsilon.
    \]
    Let $\theta_i$ denote the value on the grid $\{\frac{i}{m} - \frac{1}{2m} : i \in \{1,\dots,m\}\}$ that is closest to $\theta$ with the extra specification that in the case of ties we always round up. By our assumption of the support of $p(\cdot)$ we have that
    \begin{align*}
    |\mme[\ell_{\theta}(p(X),Y) - \ell_{\theta_i}(p(X),Y)]| & = \left|\mme\left[\left(\theta - \theta_i\right)\bone\{Y=0,p(X) > \theta\} + \left( \theta_i - \theta\right)\bone\{Y=1,p(X) \leq \theta\} \right]\right|\\
    & \leq \frac{1}{2m}.
    \end{align*}
    Similarly, we also have 
    \begin{align*}
& \left|\mme\left[\ell_{\theta}(f_{\theta,\epsilon}(X),Y) - \ell_{\theta_i}\left(f_{\theta,\epsilon}(X) -\theta + \theta_i,Y\right)\right]\right|\\
& = \left|\mme\left[\left(\theta - \theta_i \right)\bone\{Y=0,f_{\theta,\epsilon}(X) > \theta\} + \left( \theta_i - \theta\right)\bone\{Y=1,f_{\theta,\epsilon}(X) \leq \theta\} \right]\right| \leq \frac{1}{2m}.
    \end{align*}
So, putting these two facts together we find that
    \[
    \mme[\ell_{\theta}(p(X),Y) - \ell_{\theta}(f_{\theta,\epsilon}(X),Y)] \leq  \sup_{f \in \mathcal{F}}\mme[\ell_{\theta_i}(p(X),Y) - \ell_{\theta_i}(f(X),Y)] + \frac{1}{m},
    \]
    and sending $\epsilon \to 0$ gives the desired result.
\end{proof}

\section{Proofs for Section \ref{sec:algorithms}}

\subsection{Proofs for Section \ref{sec:online_omni}}\label{app:online_omni}

In this section we prove Lemma \ref{lem:min_max_opt_char} and Theorem \ref{thm:online_omni_bound}.

\begin{proof}[Proof of Lemma \ref{lem:min_max_opt_char}]
 As stated in the main text, we consider the distribution $P^* = (1-\rho^*) \delta_{\theta^*} + \rho^* \delta_{\theta^* + i/m}$ where
\begin{align*}
    & \theta^* = \sup \left\{\theta \in \left\{0,\frac{1}{m},\frac{2}{m},\dots,1\right\} :  \sum_{i=1}^{m}  q_{i} \bone\{\theta \leq \theta_i\} \geq \sum_{i=1}^{m}  q_{i} \bone\{\hat{f}_{\theta_i}(x) \leq \theta_i\}  \right\}\\
    \text{and } \  & \rho^* = \frac{\sum_{i=1}^{m} q_{i} \bone\{\theta^*  \leq \theta_i\} - \sum_{i=1}^{m} q_{i} \bone\{\hat{f}_{\theta_i}(x) \leq \theta_i\} }{q_{m\theta^* 1}},
\end{align*}
with the caveat that for ease of notation we define $q_{m+1} = 1$ so that $\rho^* = 0$ when $\theta^* = 1$. In addition, let $p_y^* := \min\{\theta^* + \frac{1}{2m},1\}$. To prove that $P^*$ is optimal it is sufficient to prove that the pair $(P^*, p_y^*)$ is a saddle point to the min-max program. To see this, first note that for any $(P,p_y)$ the optimization objective can be written as
\begin{align*}
     O(P,p_y) & :=\mme_{p \sim P,Y' \sim \text{Ber}(p_y)}\left[ \sum_{i=1}^{m} q_{i} (\ell_{\theta_i}(p,Y') - \ell_{\theta_i}(\hat{f}_{\theta_i}(x),Y'))\right]\\
& = \mme_{p \sim P,Y' \sim \text{Ber}(p_y)}\bigg[ \sum_{i=1}^{m} q_{i}\bigg(\theta_i \bone\left\{p>\theta_i, Y' = 0\right\} + \left(1-\theta_i\right) \bone\left\{p\leq \theta_i, Y' = 1\right\}\\
& \hspace{3cm} - \theta_i \bone\left\{\hat{f}_{\theta_i}(x)>\theta_i, Y' = 0\right\}  - \left(1-\theta_i\right) \bone\left\{\hat{f}_{\theta_i}(x) \leq \theta_i, Y' = 1 \right\} \bigg) \bigg]\\
& = \mme_{p \sim P}\bigg[  \sum_{i=1}^{m} q_{i} \bigg(\theta_i (1-p_y) \bone\left\{p > \theta_i\right\} + \left(1-\theta_i\right) p_y \bone\left\{p \leq \theta_i\right\}\\
& \hspace{3cm} - \theta_i (1-p_y) \bone\left\{\hat{f}_{\theta_i}(x)>\theta_i\right\} - \left(1-\theta_i \right) p_y \bone\left \{\hat{f}_{\theta_i}(x) \leq \theta_i\right\} \bigg) \bigg]\\
& =  \mme_{p \sim P}\left[  \sum_{i=1}^{m} q_{i} \left(p_y - \theta_i\right) \left(\bone\left\{p \leq \theta_i\right\} - \bone\left\{\hat{f}_{\theta_i}(x) \leq \theta_i\right\}\right)\right]
\end{align*}

Now, plugging in our choice of $P^*$ gives an objective value of
\begin{align*}
O(P^*,p_y) & =   \sum_{i=1}^{m} q_{i} p_y \left( \bone\left\{ \theta^* \leq \theta_i \right\} - \bone\left\{\hat{f}_{\theta_i}(x) \leq \theta_i\right\} \right) - \rho^*p_y q_{m\theta^* + + 1} \\
& \ \ \ -  \mme_{p \sim P}\left[  \sum_{i=1}^{m} q_{i}  \theta_i \left(\bone\left\{p \leq \theta_i\right\} - \bone\left\{\hat{f}_{\theta_i}(x) \leq \theta_i\right\}\right)\right]\\
& = -  \mme_{p \sim P}\left[  \sum_{i=1}^{m} q_{i}  \theta_i \left(\bone\left\{p \leq \theta_i\right\} - \bone\left\{\hat{f}_{\theta_i}(x) \leq \theta_i\right\}\right)\right],
\end{align*}
where the second equality follows immediately from our choice of $\rho^*$. Since this last expression does not depend on $p_y$, we must have that $O(P^*,p^*_y) = \max_{p_y \in [0,1]} O(P^*,p_y)$.

On the other hand, since the losses $\{\ell_{\theta_i}\}_{i=1}^{m}$ are proper we must have that at $p_y = p_y^*$, $O(P,p_y^*)$ is minimized by setting $P = \delta_{p_y^*}$. Moreover, it is easy to check that for all $i \in \{1,\dots,m\}$,
\[
\mme_{Y' \sim \text{Ber}(p_y^*)}[\ell_{\theta_i}(p_y^*,Y')] = \mme_{Y' \sim \text{Ber}(p_y^*)}[\ell_{\theta_i}(\theta^*,Y')] = \mme_{Y' \sim \text{Ber}(p_y^*)}[\ell_{\theta_i}(\theta^* + 1/m,Y')].
\]
In particular, this implies that $O(P^*,p_y^*) = O(\delta_{p_y^*},p_y^*)$ and thus that  $O(P^*,p_y^*) = \min_{P \in \Delta_m} O(P,p_y^*)$, as desired.

\end{proof}

\begin{proof}[Proof of Theorem \ref{thm:online_omni_bound}]
For ease of notation, note that in what follows all expectations treat $\hat{P}(\cdot)$ as fixed and are taken only with respect to the variables appearing in the associated subscripts. By the results of Section \ref{sec:simplification}, it is sufficient to bound
\[
\sup_{i \in \{1,\dots,m\}} \mme_{(X,Y)}[\mme_{p \sim \hat{P}(X)}[\ell_{\theta_i}(p,Y)] ] - \mme_{(X,Y)}[\ell_{\theta_i}(\hat{f}_{\theta_i}(X),Y)].
\]
Fix any $i \in \{1,\dots,m\}$. By definition of $\hat{P}$, we have that 
\begin{align*}
& \mme_{(X,Y)}[\mme_{p \sim \hat{P}(X)}[\ell_{\theta_i}(p,Y)] ] - \mme_{(X,Y)}[\ell_{\theta_i}(\hat{f}_{\theta_i}(X),Y)]\\
 =  & \frac{1}{n} \sum_{t=1}^n (\mme_{(X,Y)}[\mme_{p \sim \hat{P}_t(X)}[\ell_{\theta_i}(p,Y)] ] - \mme_{(X,Y)}[\ell_{\theta_i}(\hat{f}_{\theta_i}(X),Y)]).
\end{align*}
Now, consider the martingale 
\[
M_t(i) = \sum_{s=1}^t (\mme_{p \sim \hat{P}_s(X_s)}[\ell_{\theta_i}(p,Y_s)] - \ell_{\theta_i}(\hat{f}_{\theta_i}(X_s),Y_s)]) - (\mme_{(X,Y)}[\mme_{p \sim \hat{P}_s(X)}[\ell_{\theta_i}(p,Y)] - \ell_{\theta_i}(\hat{f}_{\theta_i}(X),Y)]).
\]
By the Azuma-Hoeffding inequality (Theorem \ref{thm:AH} below), 
\[
\sup_{i \in \{1,\dots,m\}} |M_n(i)|/n \leq O_{\mmp}(\sqrt{\log(m)/n}),
\]
and so, in particular, 
\begin{align*}
& \mme_{(X,Y)}[\mme_{p \sim \hat{P}(X)}[\ell_{\theta_i}(p,Y)] ] - \mme_{(X,Y)}[\ell_{\theta_i}(\hat{f}_{\theta_i}(X),Y)]\\
 \leq & \frac{1}{n} \sum_{s=1}^n (\mme_{p \sim \hat{P}_s(X_s)}[\ell_{\theta_i}(p,Y_s)] - \ell_{\theta_i}(\hat{f}_{\theta_i}(X_s),Y_s))  + O_{\mmp}(\sqrt{\log(m)/n}).
\end{align*}
Now, by standard regret bounds for the hedge algorithm (Theorem \ref{thm:hedge} below) the first term above is itself bounded by 
\[
\frac{1}{n} \sum_{s=1}^n \sum_{j=1}^{m} q_{j}(s) (\mme_{p \sim \hat{P}_s(X_s)}[\ell_{\theta_j}(p,Y_s)] - \ell_{\theta_j}(\hat{f}_{\theta_j}(X_s),Y_s)) + 4 \eta  + \frac{\log(m)}{n\eta},
\]
and by Lemma \ref{lem:online_omni_min_max_program_bound} we know that the first term above is non-positive. Putting all of the above inequalities together, we find that 
\[
\sup_{i \in \{1, \dots,m\}} \mme_{(X,Y)}[\mme_{p \sim \hat{P}(X)}[\ell_{\theta_i}(p,Y)] ] - \mme_{(X,Y)}[\ell_{\theta_i}(\hat{f}_{\theta_i}(X),Y)] \leq O_{\mmp}(\sqrt{\log(m)/n})+   \eta  + \frac{\log(m)}{n\eta},
\]
and plugging in our choices of $\eta$ and $m$ gives the desired result.
\end{proof}

\subsection{Proofs for Section \ref{sec:direct_omni}}\label{app:direct_omni}

In this section we prove Lemma \ref{lem:merge_performance} and Theorem \ref{thm:direct_omni_bound}. We begin by stating a more detailed version of our merge algorithm which defines a number of additional quantities that will be useful in the proof. Most crucially, we use $A_{h,\cdot}$ and $A_{l,\cdot}$ to denote the sets on which $\hat{p}_{m}(x) = \hat{p}_{h}(x)$ and  $\hat{p}_{m}(x) = \hat{p}_l(x)$, and we use $ \{\theta_{h,0}^s,\dots,\theta_{h,k_h}^s\}$ and $\{\theta_{l,0}^s,\dots,\theta_{l,k_l}^s\}$ to denote the sets of parameters where the algorithm switches direction (i.e.~swaps from examining parameters in $\Theta_h$ to examining parameters in $\Theta_l$ and vice versa).

\begin{algorithm}[H]
\KwData{Predictors $\hat{p}_h$ and $\hat{p}_l$, sets $\Theta_{h}> \Theta_l$, hyperparameter $\epsilon$.}
$\theta_{l,0}^s = \theta_{h,0}^s = \theta_{h,0} =  -1$\;
$\theta_{l,0} = 1$\;
$k_l = k_h = 0$\;
$t = 1$\;
$A_{l,1} = \emptyset$\;
$A_{h,1} = \mathcal{X}$\;
$\theta_{l,1} = \max\Theta_l $\;
$\theta_{h,1} = \min\Theta_h $\;
$\textup{dir}(1) = \textup{low}$\;
\While{$\theta_l \neq -\infty$, $\theta_{h} \neq \infty$}{
$E = \bone\{x : \hat{p}_h(x) > \theta_{h,t},\ \hat{p}_l(x) \leq \theta_{l,t}\}$\;
\uIf{$\textup{dir}(t) =  \textup{low}$}{
    \uIf{$\hat{\mme}_n[(\ell_{\theta_{l,t}}(0,Y)- \ell_{\theta_{l,t}}(1,Y))\bone\{X \in E\}] < -\epsilon$}{
            $A_{l,t+1} = A_{l,t} \cup  E$\;
        $A_{h,t+1} = A_{h,t} \setminus E$\; 
        $\theta_{h,t+1} = \min\{\theta \in 
        \Theta_{h} : \theta > \theta_{h,t}\}$\;
        $\textup{dir}(t+1) = \textup{high}$\;
        $k_l = k_l + 1$\;
        $\theta_{l,k_l}^s = \theta_{l,t}$\;
    }\Else{
            $\theta_{l,t+1} = \max\{\theta \in 
        \Theta_l : \theta < \theta_{l,t}\}$\;
        $\textup{dir}(t+1) = \textup{low}$\;
    }
    
}\Else{
    \uIf{$\hat{\mme}_n[(\ell_{\theta_{h,t}}(1,Y)- \ell_{\theta_{h,t}}(0,Y))\bone\{X \in E\}] < -\epsilon$ }{
        $A_{h,t+1} = A_{h,t} \cup E$\;
        $A_{l, t+1} = A_{l,t} \setminus  E$\; 
        $\theta_{l,t+1} = \max\{\theta \in 
        \Theta_l : \theta < \theta_{l,t}\}$\;
        $\textup{dir}(t+1) = \textup{low}$\;
        $k_{h} = k_{h} + 1$\;
        $\theta_{h,k_{h}}^s = \theta_{h,t}$\;
    }\Else{
        $\theta_{h,t+1} = \min\{\theta \in \Theta_{h} : \theta > \theta_{h,t}\}$\;
        $\textup{dir}(t+1) = \textup{high}$\;
    }
}
$t = t+1$\;
}
    \Return $\hat{p}_{m}(X) =   \hat{p}_l(X) \bone\{X \in A_{l}\} + \hat{p}_{h}(X) \bone\{X \in A_{h}\} $
\caption{Detailed merge procedure}\label{alg:merge_detailed}
 \end{algorithm}

We will now prove Lemma \ref{lem:merge_performance} using a sequence of sublemmas. As a final piece of notation, we let $c_{h,t} = |\{s < t : \text{dir}(s) = \text{high},\ \text{dir}(s+1) = \text{low}\}|$ and $c_{\ell,t} = |\{s < t : \text{dir}(s) = \text{low},\ \text{dir}(s+1) = \text{high}\}|$ denote the number of times the direction switches from high (resp. low) to low (resp. high) before timestep $t$. Our first lemma characterizes the structure of the sets $A_{h,t}$ and $A_{l,t}$.

\begin{lemma}\label{lem:merge_set_char}
    Let $\Theta_h > \Theta_l$ be finite sets and assume that $\hat{p}_h$ and $\hat{p}_l$ take values in $[0,1]$. For each timestep $t$ on which $\textup{dir}(t) = \textup{high}$,
    \begin{equation}
    \begin{gathered}
    A_{h,t} = \bigcup_{i=1}^{c_{h,t}} \{x : \theta^s_{h,i-1} < \hat{p}_{h}(x) \leq \theta^s_{h,i},\ \hat{p}_l(x) > \theta^s_{\ell,i}  \} \cup \{x : \hat{p}_{h}(x) > \theta^s_{h,c_{h,t}},\ \hat{p}_l(x) > \theta^s_{\ell,c_{\ell,t}}\},\\
    A_{l,t} = \bigcup_{i=1}^{c_{\ell,t}} \{x : \theta^s_{\ell,i} < \hat{p}_l(x) \leq \theta^s_{\ell,i-1},\ \hat{p}_{h}(x) \leq \theta^s_{h,i-1}  \} \cup \{x : \hat{p}_l(x) \leq \theta^s_{\ell,c_{\ell,t}}\}.
    \end{gathered}
    \end{equation}
Moreover, for each timestep $t$ on which $\textup{dir}(t) = \textup{low}$,
 \begin{equation}
    \begin{gathered}
    A_{h,t} = \bigcup_{i=1}^{c_{h,t}} \{x : \theta^s_{h,i-1} < \hat{p}_{h}(x) \leq \theta^s_{h,i},\  \hat{p}_l(x) > \theta^s_{\ell,i}  \} \cup \{x : \hat{p}_{h}(x) > \theta^s_{h,c_{h,t}}\},\\
    A_{l,t} = \bigcup_{i=1}^{c_{\ell,t}} \{x : \theta^s_{\ell,i} < \hat{p}_l(x) \leq \theta^s_{\ell,i-1},\ \hat{p}_{h}(x) \leq \theta^s_{h,i-1}  \} \cup \{x : \hat{p}_{h}(x) \leq \theta^s_{h,c_{h,t}},\ \hat{p}_l(x) \leq \theta^s_{\ell,c_{\ell,t}}\}.
    \end{gathered}
    \end{equation}
\end{lemma}
\begin{proof}
    We proceed by induction on $t$. The base case of $t=0$ is immediate. For the induction step, suppose for simplicity that the result holds at timestep $t$ and $\textup{dir}(t) = \text{low}$ (the case where $\textup{dir}(t) = \text{high}$ is identical). If $\textup{dir}(t+1) = \textup{dir}(t) = \text{low}$ there is nothing to prove. So, suppose $\textup{dir}(t+1) = \text{high}$. Then,
    \begin{align*}
    & A_{h,t+1}  = A_{h,t} \setminus \{x : \hat{p}_{h}(x) > \theta_{h,t} ,\ \hat{p}_l(x) \leq \theta_{l,t}\} \\
    & = \bigcup_{i=1}^{c_{h,t}} \{x : \theta^s_{h,i-1} < \hat{p}_{h}(x) \leq \theta^s_{h,i},\ \hat{p}_l(x) > \theta^s_{\ell,i}  \} \\
    & \hspace{5cm} \cup \{x : \hat{p}_{h}(x) > \theta^s_{h,c_{h,t}}\} \setminus \{x : \hat{p}_{h}(x) > \theta_{h,t} ,\ \hat{p}_l(x) \leq \theta_{l,t}\}.
    \end{align*}
    Now, by definition $c_{h,t+1} = c_{h,t}$, $\theta^s_{c_{h,t}} = \theta_{h,t}$, $c_{\ell,t+1} = c_{\ell,t} + 1$, and $\theta^s_{\ell,c_{\ell,t+1}} = \theta_{l,t}$. So, the above can immediately be re-written as
    \[
    \bigcup_{i=1}^{c_{h,t+1}} \{x : \theta^s_{h,i-1} < \hat{p}_{h}(x) \leq \theta^s_{h,i},\ \hat{p}_l(x) > \theta^s_{\ell,i}  \} \cup \{x : \hat{p}_{h}(x) > \theta^s_{h,c_{h,t+1}},\ \hat{p}_l (x)> \theta^s_{\ell,c_{\ell,t+1}}\},
    \]
    as desired. Moreover, note that by construction $c_{\ell,t+1} = c_{h,t} + 1$. So, we also have that
    \begin{align*}
        & A_{l,t+1} = A_{l,t} \cup \{x : \hat{p}_{h}(x) > \theta_{h,t} ,\ \hat{p}_l(x) \leq \theta_{l,t}\} \\
        & = \bigcup_{i=1}^{c_{\ell,t}} \{x : \theta^s_{\ell,i} < \hat{p}_l(x) \leq \theta^s_{\ell,i-1},\ \hat{p}_{h}(x) \leq \theta^s_{h,i-1}  \}\\
        & \hspace{3cm} \cup \{x: \hat{p}_{h}(x) \leq \theta^s_{h,c_{h,t}},\ \hat{p}_l(x) \leq \theta^s_{\ell,c_{\ell,t}}\} \cup \{x : \hat{p}_{h}(x) > \theta_{h,t} ,\ \hat{p}_l(x) \leq \theta_{l,t}\} \\
                & = \bigcup_{i=1}^{c_{\ell,t}} \{x : \theta^s_{\ell,i} < \hat{p}_l(x) \leq \theta^s_{\ell,i-1},\ \hat{p}_{h}(x) \leq \theta^s_{h,i-1}  \}\\
        & \hspace{3cm} \cup \{x: \hat{p}_{h}(x) \leq \theta^s_{h,c_{h,t}},\ \hat{p}_l(x) \leq \theta^s_{\ell,c_{\ell,t}}\} \cup \{x : \hat{p}_{h}(x) > \theta^s_{h,c_{h,t}} ,\ \hat{p}_l(x) \leq \theta^s_{\ell,c_{\ell,t+1}}\} \\
        & = \bigcup_{i=1}^{c_{\ell,t+1}} \{x : \theta^s_{\ell,i} < \hat{p}_l(x) \leq \theta^s_{\ell,i-1},\ \hat{p}_{h}(x) \leq \theta^s_{h,i-1}  \} \cup \{x : \hat{p}_l(x) \leq \theta^s_{\ell,c_{\ell,t+1}}\}.
    \end{align*}
\end{proof}

 Our next lemma upperbounds the loss of the ensembled predictor computed by the Merge procedure at each iteration of the algorithm.

\begin{lemma}\label{lem:merge_iter_performance}
   Let $\Theta_h > \Theta_l$ be finite subsets of $[0,1]$ and assume that $\hat{p}_h$ takes values in $(\max \Theta_l,1]$  and $\hat{p}_l$ take values in $[0,\min \Theta_h)$. For all $t$ let
    \[
    \hat{p}_{m,t}(x) =   \hat{p}_l(x)\bone\{x \in A_{l,t}\}+ \hat{p}_{h}(x)\bone\{x \in A_{h,t}\}.
    \]
    Fix $\epsilon > 0$ and suppose that,
    \[
    \sup_{\theta_{h} \in \Theta_{h}, \theta_l \in \Theta_l, \theta \in \{\theta_{h}, \theta_l\}} \left|(\hat{\mme}_n - \mme)[(\ell_{\theta}(1,Y) - \ell_{\theta}(0,Y))\bone\{\hat{p}_{h}(X) > \theta_{h},\ \hat{p}_l(X) \leq \theta_l\}]\right| \leq \epsilon.
    \]
    Then, for all $t$ such that $\textup{dir}(t) = \textup{high}$ we have
    \[
    \sup_{\theta \in \Theta_{h} : \theta < \theta_{h,t}} \mme[\ell_{\theta}(\hat{p}_{m,t}(X),Y) - \ell_{\theta}(\hat{p}_{h}(X),Y)] \leq 2\epsilon \  \text{ and } \  \sup_{\theta \in \Theta_l} \mme[\ell_{\theta}(\hat{p}_{m,t}(X),Y) - \ell_{\theta}(\hat{p}_l(X),Y)] \leq 2\epsilon.
    \]
    Similarly, for all $t$ such that $\textup{dir}(t) = \textup{low}$ we have
    \[
        \sup_{\theta \in \Theta_{h}} \mme[\ell_{\theta}(\hat{p}_{m,t}(X),Y) - \ell_{\theta}(\hat{p}_{h}(X),Y)] \leq 2\epsilon \  \text{ and } \  \sup_{\theta \in \Theta_l: \theta > \theta_{l,t}} \mme[\ell_{\theta}(\hat{p}_{m,t}(X),Y) - \ell_{\theta}(\hat{p}_l(X),Y)] \leq 2\epsilon,
    \]
    where all of the expectations above are taken only over the randomness in $(X,Y)$ with $\hat{p}_{m,t}$ held fixed.
\end{lemma}
\begin{proof}
    We prove this by induction. The base case of $t=0$ is immediate. For the inductive step, suppose the result holds at timestep $t$. Assume for simplicity that $\textup{dir}(t) = \textup{high}$ (the case $\textup{dir}(t) = \textup{low}$ is identical). There are two cases.

    \textbf{Case 1, $\textup{dir}(t+1) = \textup{high}$:} In this case the predictor does not change. Thus, to obtain the desired result we just need to show that
    \[
    \mme[\ell_{\theta_{h,t}}(\hat{p}_{m,t}(X),Y) - \ell_{\theta_{h,t}}(\hat{p}_{h}(X),Y)] \leq 2\epsilon.
    \]
    By Lemma \ref{lem:merge_set_char}, we have
    \begin{align*}
        & \mme[\ell_{\theta_{h,t}}(\hat{p}_{m,t}(X),Y) - \ell_{\theta_{h,t}}(\hat{p}_{h}(X),Y)]\\
        & = \mme[(\ell_{\theta_{h,t}}(0,Y) - \ell_{\theta_{h,t}}(1,Y))\bone\{X \in A_{l,t},\ \hat{p}_{h}(X) > \theta_{h,t}, \ \hat{p}_l(X) \leq \theta_{h,t}\}]\\
        & = \mme[(\ell_{\theta_{h,t}}(0,Y) - \ell_{\theta_{h,t}}(1,Y))\bone\{\hat{p}_l(X) \leq \theta^s_{c_{\ell,t}},\ \hat{p}_{h}(X) > \theta_{h,t}\}].
    \end{align*}
    Now, by construction, $\theta^s_{c_{\ell,t}} = \theta_{l,t}$. So, the above is quantity is exactly equal to 
    \begin{align*}
        & \mme[(\ell_{\theta_{h,t}}(0,Y) - \ell_{\theta_{h,t}}(1,Y))\bone\{\hat{p}_l(X) \leq \theta_{l,t},\ \hat{p}_{h}(X) > \theta_{h,t}\}]\\
        & = ({\mme} -\hat{\mme}_n)[(\ell_{\theta_{h,t}}(0,Y) - \ell_{\theta_{h,t}}(1,Y))\bone\{\hat{p}_l(X) \leq \theta_{l,t},\ \hat{p}_{h}(X) > \theta_{h,t}\}]\\
        & \ \ \ \ \ \ + \hat{\mme}_n[(\ell_{\theta_{h,t}}(0,Y) - \ell_{\theta_{h,t}}(1,Y))\bone\{\hat{p}_l(X) \leq \theta_{l,t},\ \hat{p}_{h}(X) > \theta_{h,t}\}]\\
        & \leq 2\epsilon,
    \end{align*} 
    where to obtain the last line we recall that $\text{dir}(t) = \text{dir}(t+1) = \text{high}$ and thus the empirical expectation in the second term must be at most $\epsilon$.

    \textbf{Case 2, $\textup{dir}(t+1) = \textup{low}$:} Now, by construction, in order to have $\textup{dir}(t) = \textup{high}$ and $\textup{dir}(t+1) = \textup{low}$ we must have that
    \begin{align*}
\hat{\mme}_n[(\ell_{\theta_{h,t}}(1,Y) - \ell_{\theta_{h,t}}(0,Y))\bone\{\hat{p}_l(X) \leq \theta_{l,t},\ \hat{p}_{h}(X) > \theta_{h,t}\}] < -\epsilon.
    \end{align*}
Notably, it follows immediately that 
\[
\hat{\mme}_n[(\ell_{\theta}(1,Y) - \ell_{\theta}(0,Y))\bone\{\hat{p}_l(X) \leq \theta_{l,t},\ \hat{p}_{h}(X) > \theta_{h,t}\}] < -\epsilon, \ \forall \theta \leq \theta_{h,t}.
\]
We will use this fact multiple times in the calculations that follow.

We consider a series of sub-cases. First, consider the case where $\theta \in  \{\theta' \in \Theta_l: \theta' \geq \theta_{l,t}\}$. By the induction hypothesis, 
    \begin{align*}
        & \mme[\ell_{\theta}(\hat{p}_{m,t+1}(X),Y) - \ell_{\theta}(\hat{p}_l(X),Y)] \leq \mme[\ell_{\theta}(\hat{p}_{m,t+1}(X),Y) - \ell_{\theta}(\hat{p}_{m,t}(X),Y)] + 2\epsilon\\
        & = \mme[(\ell_{\theta}(1,Y) - \ell_{\theta}(0,Y))\bone\{\hat{p}_{h}(X) > {\theta}_{h,t},\ \hat{p}_l(X) \leq {\theta}_{\ell,t}\}] + 2\epsilon\\
        & \leq (\mme-\hat{\mme}_n)[(\ell_{\theta}(1,Y) - \ell_{\theta}(0,Y))\bone\{\hat{p}_{h}(X) > {\theta}_{h,t},\ \hat{p}_l(X) \leq {\theta}_{\ell,t}\}]\\
        & \hspace{2cm} + \hat{\mme}_n[(\ell_{\theta}(1,Y) - \ell_{\theta}(0,Y))\bone\{\hat{p}_{h}(X) > {\theta}_{h,t},\ \hat{p}_l(X) \leq {\theta}_{\ell,t}\}] + 2\epsilon \\
        & \leq \epsilon - \epsilon + 2\epsilon = 2\epsilon.
    \end{align*}
    On the other hand, for $\theta \geq \theta_{h,t}$ we have that $\hat{p}_{m,t+1}(x) > \theta \iff \hat{p}_{h}(x) > \theta$ (recall Lemma \ref{lem:merge_set_char} and that $\theta_{h,c_{h,t}}^s = \theta_{h,t}$) and thus,
    \[
    \mme[\ell_{\theta}(\hat{p}_{m,t+1}(X),Y) - \ell_{\theta}(\hat{p}_{h}(X),Y)] = 0.
    \]
    Finally, for $\theta \in \{ \theta' \in \Theta_{h}:  \theta' < \theta_{h,t}\}$ we have
    \begin{align*}
         \mme[\ell_{\theta}(\hat{p}_{m,t+1}(X),Y) - \ell_{\theta}(\hat{p}_{h}(X),Y)] & \leq \mme[\ell_{\theta}(\hat{p}_{m,t+1}(X),Y) - \ell_{\theta}(\hat{p}_{m,t}(X),Y)]  + 2\epsilon\\
        & = \mme[(\ell_{\theta}(1,Y) - \ell_{\theta}(0,Y))\bone\{\hat{p}_{h}(X) > {\theta}_{h,t},\ \hat{p}_l(X) \leq {\theta}_{\ell,t}\}]  + 2\epsilon\\
        & \leq 2\epsilon,
    \end{align*}
    as above.
\end{proof}

 We are now ready to prove Lemma \ref{lem:merge_performance} which follows as an almost immediate corollary of Lemma \ref{lem:merge_iter_performance}.

 \begin{proof}[Proof of Lemma \ref{lem:merge_performance}] By Hoeffding's inequality we have that 
 \[
 \sup_{\theta_{h} \in \Theta_{h}, \theta_l \in \Theta_l, \theta \in \{\theta_{h}, \theta_l\}} \left|(\hat{\mme}_n - \mme)[(\ell_{\theta}(1,Y) - \ell_{\theta}(0,Y))\bone\{\hat{p}_{h}(X) > \theta_{h},\ \hat{p}_l(X) \leq \theta_l]\right| = O_{\mmp}\left(\sqrt{\frac{\log(|\Theta_h|\cdot |\Theta_l|)}{n}} \right).
 \]
Plugging this fact into the statement of Lemma \ref{lem:merge_iter_performance} and taking $t$ to be the last time-step of Algorithm \ref{alg:merge_detailed} gives the desired result.
 \end{proof}

With the above lemmas in hand the proof of Theorem \ref{thm:direct_omni_bound} is immediate.

\begin{proof}[Proof of Theorem \ref{thm:direct_omni_bound}]
This result follows immediately from combining Lemma \ref{lem:merge_performance} with the results of Section \ref{sec:simplification} and adding up the cumulative error over all $\log_2(m)$ rounds of Algorithm \ref{alg:ensemb_main}.
\end{proof}

\section{Additional details for the sales forecasting example}\label{sec:app_sales_details}

For our sales forecasting example in Section \ref{sec:sales} we need to compute the forecasted probability of observing a non-zero number of sales given a predicted set of quantiles. Formally, let $Y_c \in \mmr$ denote the number of sales of an item on a given day at a given Walmart location. Let $0 < \tau_1<\cdots < \tau_k < 1$ denote a set of levels and $\hat{q}^{\tau_1} \leq \cdots \leq \hat{q}^{\tau_k}$ denote a corresponding set of quantile estimates. Then, for any $x \in \mmr$ we define an estimate of the cumulative distribution function of $Y_c$ as the linear interpolation,
\[
\hat{\mmp}(Y_c \leq x) = \begin{cases}
 1,\ x >= \tau_k,\\
 0,\ x < \tau_1,\\
 \tau_{i-1} + \frac{\tau_i - \tau_{i-1}}{\hat{q}^{\tau_i} - \hat{q}^{\tau_{i-1}}}(x-  \hat{q}^{\tau_{i-1}}), \hat{q}^{\tau_{i-1}} \leq x < \hat{q}^{\tau_i}.
\end{cases}
\]

\begin{figure}[ht]
    \centering\includegraphics[width=\textwidth]{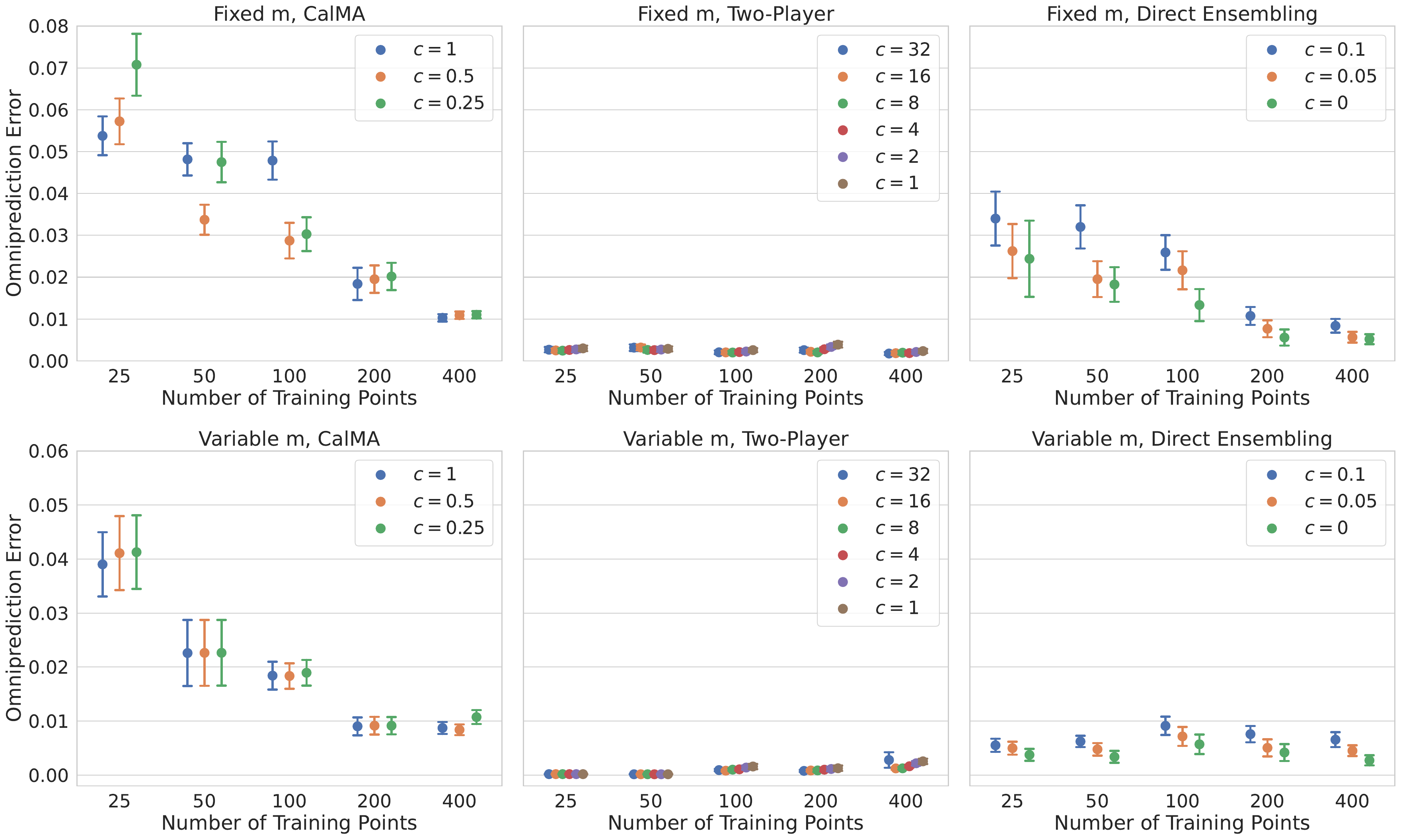}
    \caption{Omniprediction error of the calibrated multiaccuracy (left panels), two-player game based (center panels), and direct ensembling (right panels) methods across various sample sizes with $m = 16$ fixed (top row) or chosen variably as $m = 2^{\lfloor \log_2(\sqrt{n}) \rfloor}$ (bottom row) as the scaling constant $c$ varies on the M5 sales forecasting dataset. Dots and error bars show means and standard errors obtained by evaluating the omniprediction error over 2000 test points for each of 20 draws of the training data.    }\label{fig:sales_var_hyper}
\end{figure}

We conclude this section with Figure \ref{fig:sales_var_hyper} which displays the results of our sales forecasting experiments for varying hyperparameter values.

\section{Proofs for Section \ref{sec:other_targets} }

In this section we prove Proposition \ref{prop:general_loss_decomp}.

 \begin{proof}[Proof of Proposition \ref{prop:general_loss_decomp}]
    The statement given in Proposition \ref{prop:general_loss_decomp} is  a slight variant of Corollary 9 of \citet{Steinwart2014}. In particular, we have assumed that the losses under consideration are strictly proper, while \citet{Steinwart2014} instead assumes that the losses  are order sensitive. More precisely, they restrict to losses $\ell^T$ such that for all distributions $P \in \mathcal{P}$ and all $t_1,t_2 \in \text{Image}(T)$ such that either $t_2 < t_1 < T(P)$ or $T(P) < t_1 < t_2$,
    \[
    \mme_P[\ell^T(t_1,Y)] < \mme_P[\ell^T(t_2,Y)].
    \]
    We show here that this latter condition is implied by strict propriety. 
    
    Let $\ell^T$ be a strictly proper loss for $T$ and $t_1,t_2 \in \text{Image}(T)$ be such that either $t_2 < t_1 < T(P)$ or $T(P) < t_1 < t_2$. Let $P_1$ and $P_2$ be such that $T(P_1) = t_1$ and $T(P_2) = t_2$. By the continuity of $T$, there exists $\lambda \in (0,1)$ such that $T(\lambda P_2 + (1-\lambda) P) = T(P_1)$. Moreover, since $\ell^T$ is strictly proper we must have that 
    \begin{align*}
    \lambda \mme_{P_2}[\ell^T(t_1,Y)] + (1-\lambda) \mme_P[\ell^T(t_1,Y)] & = \mme_{\lambda P_2 + (1-\lambda) P}[\ell^T(t_1,Y)]\\
    & < \mme_{\lambda P_2 + (1-\lambda) P}[\ell^T(t_2,Y)]  =  \lambda \mme_{P_2}[\ell^T(t_2,Y)] + (1-\lambda) \mme_P[\ell^T(t_2,Y)],
    \end{align*}
    and so in particular,
    \[
    (1-\lambda) (\mme_P[\ell^T(t_2,Y)] - \mme_P[\ell^T(t_1,Y)]) > \lambda(\mme_{P_2}[\ell^T(t_1,Y)] - \mme_{P_2}[\ell^T(t_2,Y)]) > 0,
    \]
    as desired.

\end{proof}

\section{Auxiliary results}

In this section we state a few results from prior work that were used in the proofs from the previous sections. We begin by recalling the regret bound for the well-known hedge algorithm for learning from expert advice \citep{Vovk1990, Littlestone1994, Freund1997}.

\begin{theorem}[Regret of Hedge (e.g., Theorem 1.5 of \citet{Hazan2019})]\label{thm:hedge}
Consider an online learning problem with $m$ experts receiving bounded losses $\{\ell_{t,i}\}_{1 \leq  i \leq m, 1 \leq t \leq T}$ with $\sup_{1 \leq i \leq m, 1 \leq t \leq T} \ell_{t,i} \leq B$. Suppose that at time step $t$ we make the same prediction as expert $i$ with probability 
\[
q_{t,i} := \frac{ \exp(-\eta \sum_{s<t} \ell_{s,i})}{\sum_{j=1}^m\exp(-\eta \sum_{s<t} \ell_{s,j}) },
\]
for some $\eta > 0$. Then,
\[
\sum_{t=1}^T \mme_{I \sim q_{t}}[\ell_{t,I}] \leq \min_{1 \leq i \leq m} \sum_{t=1}^T \ell_{t,i} + \eta T B^2 + \frac{\log(M)}{\eta}.
\]
\end{theorem}

We next recall the well-known Azuma-Hoeffding inequality \citep{Hoeffding1963, Azuma1967}.

\begin{theorem}[Azuma-Hoeffding inequality (e.g., Theorem 9.7 of \citet{Hazan2019})]\label{thm:AH} 
Let $\{X_t\}_{t=1}^T$ be a martingale with bounded differences $\mmp(|X_t - X_{t-1}| \leq B) = 1$, $\forall 2 \leq t \leq T$. Then, for all $c \in \mmr$,
\[
\mmp(|X_T - \mme[X_T]| \geq c) \leq 2\exp\left( -\frac{c^2}{2B^2 T} \right).
\]
    
\end{theorem}

\end{document}